\newif\ifproofs\proofsfalse

\documentclass{article}
\usepackage{hyperref}
\usepackage{url}
\usepackage{color-edits}
\usepackage{comment}
\usepackage{graphicx}
\usepackage{amsmath,amssymb}
\usepackage{amsthm}
\usepackage{fullpage}
\usepackage{bm}
\usepackage{mathrsfs}

\graphicspath{{figures/}}
\usepackage{xcolor}
\usepackage[square, comma, numbers,sort&compress]{natbib}

\newtheorem{thm}{Theorem}
\newtheorem{lemma}{Lemma}
\newtheorem{definition}{Definition}
\newtheorem{cor}{Corollary}
\newtheorem{rmk}{Remark}

\usepackage{algorithm}
\usepackage{algpseudocode}
\algtext*{EndWhile}
\algtext*{EndFor}
\algtext*{EndIf}
\algnewcommand\algorithmicinput{\textbf{Input:}}
\algnewcommand\INPUT{\item[\algorithmicinput]}
\algnewcommand\algorithmicoutput{\textbf{Output:}}
\algnewcommand\OUTPUT{\item[\algorithmicoutput]}

\newcommand{\E}{\mathop{\mathbb{E}}}
\renewcommand{\Pr}{\mathrm{Pr}}
\DeclareMathOperator*{\argmax}{arg\,max}
\newcommand{\like}{\mathcal{L}}
\newcommand{\N}{\mathcal{N}}
\newcommand{\bigo}{\mathcal{O}}
\newcommand{\cdf}{\mathcal{F}}

\newcommand{\alloc}{\mathcal{A}} 
\renewcommand{\c}{\textbf{c}}
\newcommand{\ci}{c_i}
\newcommand{\cit}{c_i^t}
\newcommand{\C}{\mathcal{C}}
\newcommand{\caught}{o} 
\newcommand{\caughti}{\caught_i}
\newcommand{\caughtione}{\caught^1_i}
\newcommand{\caughtit}{\caught^t_i} 
\newcommand{\caughtis}{\caught^s_i}
\newcommand{\caughtt}{\caught^t}
\newcommand{\caughts}{\caught^s}
\newcommand{\caughtitmone}{\caught^{t-1}_i}
\DeclareMathOperator{\disc}{disc}
\newcommand{\f}{f}
\newcommand{\hvec}{\textbf{h}}
\newcommand{\hist}{h}
\newcommand{\histt}{\hvec^t}
\newcommand{\histit}{\hist^t_i}
\newcommand{\m}{m}
\newcommand{\R}{\mathcal{G}}
\newcommand{\T}{\mathcal{T}}
\newcommand{\util}{\chi}
\renewcommand{\v}{\textbf{v}}
\newcommand{\vsingle}{v}
\newcommand{\vt}{\v^t}
\newcommand{\vi}{v_i}
\newcommand{\vj}{v_j}
\newcommand{\vit}{v_i^t}
\newcommand{\vione}{\vi^{1}}
\newcommand{\vitmone}{\vi^{t-1}}
\newcommand{\V}{\mathcal{V}}
\newcommand{\vopt}{\textbf{w}}
\newcommand{\lambvec}{\bm{\lambda}}

\title{Fair Algorithms for Learning in Allocation Problems}
\author{
Hadi Elzayn, Shahin Jabbari,  Christopher Jung, Michael Kearns\\
Seth Neel,  Aaron Roth,  Zachary Schutzman\\ \\
University of Pennsylvania
}

\begin{document}	
\maketitle
\begin{abstract}
Settings such as lending and policing
can be modeled by a centralized agent allocating a scarce 
resource (e.g. loans or police officers) amongst several groups,
in order to maximize some objective (e.g. loans given that are repaid, or
criminals that are apprehended).
Often in such problems {\em fairness\/} is also a concern. One natural 
notion of fairness, based on general principles of {\em equality of opportunity\/},
asks that conditional on an individual being a candidate for the
resource in question, the probability of actually receiving it is approximately independent
of the individual's group. For example, in lending this would mean that equally
creditworthy individuals in different racial groups have roughly equal chances of
receiving a loan. In policing it would mean that two individuals committing the same
crime in different districts would have roughly equal chances of being arrested.

In this paper, we formalize this general notion of fairness for allocation problems
and investigate its algorithmic consequences. Our main technical results
include an efficient learning algorithm that converges to an optimal fair allocation 
even when the allocator does not know the frequency of candidates 
(i.e. creditworthy individuals or criminals) in each group. This algorithm operates in a 
{\em censored\/} feedback model in which only the number of candidates who received 
the resource in a given allocation can be observed, 
rather than the true number of candidates in each group. This
models the fact that we do not learn the creditworthiness of individuals we do not give loans to 
and do not learn about crimes committed if the police presence in a district is low.

As an application of our framework and algorithm, 
we consider the \emph{predictive policing} problem, in which the resource being 
allocated to each group is the number of police officers assigned to each district. 
The learning algorithm is trained on arrest 
data gathered from its own deployments on previous days, 
resulting in a potential feedback loop that our algorithm provably overcomes.
In this case, the fairness constraint asks that 
the probability that an individual who has committed a crime 
is arrested should be independent of the district in which they live. 
We empirically investigate the performance of our learning algorithm on the 
\emph{Philadelphia Crime Incidents} dataset.
\end{abstract}
\section{Introduction}
\label{sec:intro}
The bulk of the literature on algorithmic fairness has focused on classification and regression 
problems (see e.g.~\cite{HardtPS16, JosephKMR16,KleinbergMR17,
JabbariJKMR17,DworkHPRZ12,ZafarVGG17,Corbett-DaviesP17,srebro,Zem13,
LiuDRSH18, Chierichetti17, CaldersKKAZ13, BerkHJKR18, BerkHJJKMNR17} for a collection
of recent work), but fairness concerns also 
arise naturally in many resource allocation settings. Informally, a resource allocation problem 
is one in which there is a limited supply of some \emph{resource} to be distributed across 
multiple groups with differing needs. Resource allocation problems arise in financial 
applications (e.g. allocating loans), disaster response (allocating aid), and many other 
domains --- but the primary example that we will focus on in this paper is policing. In the 
predictive policing problem, the resource to be distributed is police officers, which can 
be dispatched to different districts. Each district has a different crime distribution, and the 
goal (absent additional fairness constraints) might be to maximize the number of crimes 
caught.\footnote{We understand that policing has many goals
besides simply apprehending criminals, including preventing crimes in the
first place, fostering healthy community relations, and generally promoting
public safety. But for concreteness and simplicity we consider the limited
objective of apprehending criminals.}

Of course, fairness concerns abound in this setting, and recent work 
(see e.g.~\cite{LumI16, EnsignFNSV18, EnsignFNSV182}) has highlighted the 
extent to which algorithmic allocation might exacerbate those concerns. 
For example, \citet{LumI16} show that if predictive policing algorithms such 
as PredPol are trained using past arrest data to predict future crime, then pernicious 
feedback loops can arise, which misestimate the true crime rates in certain districts, 
leading to an overallocation of police.\footnote{Predictive policing algorithms are often 
proprietary, and it is not clear whether in deployed systems, arrest data (rather than 911 
reported crime) is used to train the models.} Since the communities that \citet{LumI16} 
showed to be overpoliced on
a relative basis were primarily poor and minority, this is especially concerning from a 
fairness perspective. In this work, we study algorithms that both avoid this kind of 
under-exploration and can incorporate quantitative fairness constraints.

In the predictive policing setting,~\citet{EnsignFNSV18} implicitly consider an allocation to 
be \emph{fair} if
police are allocated across districts in direct proportion to the district's crime rate; generally 
extended, this definition asks that units of a resource are allocated according to the group's 
share of the total candidates for that resource. In our work, we study a different notion of allocative 
fairness that has a similar motivation to the notion of \emph{equality of opportunity} proposed 
by \citet{HardtPS16} in classification settings. Informally speaking, it asks that the probability that a 
candidate for a resource be allocated a resource should be independent of his group. In the predictive 
policing setting, it asks that conditional on committing a crime, the probability that an individual is 
apprehended should not depend on the district in which they commit the crime.

To illustrate that our notions of fairness do not depend on whether individuals 
would prefer to receive or not receive the good, we highlight another 
setting in which allocative fairness is a natural concern: hiring.\footnote{\citet{dwork2018fairness} 
consider such a setting under different 
fairness notions and with different research questions in mind.} Suppose a company wishes to
 recruit machine learning programmers by advertising on a social media platform. Many such 
 platforms offer the ability to advertise to different demographics of users and charge by the 
 number of times the advertisement is shown to different users (i.e., the number of \emph{impressions}); 
 a fixed advertising budget can then be viewed as a number 
 of impressions to allocate. Depending on how well the platform can identify programmers 
 within each demographic, the ad may be shown to a higher or lower number of programmers. 
 In this setting, our notion of allocative fairness asks that the probability a programmer is 
 exposed to the hiring ad (and thus, receives the opportunity to apply for a job) does not 
 depend on the programmer's demographic, and the allocation problem is to maximize the number 
 of programmers reached via the choice of impressions across each demographic, subject to fairness 
 constraints. 
	 
\subsection{Our Results}
To define the extent to which an allocation satisfies our fairness constraint, we must
model the specific mechanism by which resources deployed to a particular group 
reach their intended targets.
We study two such \emph{discovery models}, and we view the explicit framing of this 
modeling step as
one of the contributions of our work; the implications of a fairness constraint depend strongly
on the details of the discovery model, and choosing one is an important step in making 
one's assumptions transparent.

We study two discovery models which capture two extremes of targeting ability.
In the \emph{random} discovery model, however many units of the resource are allocated 
to a given group,
all individuals within that group are equally likely to be assigned a unit,
regardless of whether they are a candidate for the resource or not.
In other words, the probability that a candidate receives a resource is equal to the ratio of the
number of units of the resource assigned to his group to the size of his group ({\em independent\/} 
of the number of candidates in the group).

At the other extreme, in the \emph{precision} discovery model, units of the resource are given
only to actual candidates within a group, as long as there is sufficient supply of the resource.
In other words, the probability that a candidate receives a resource is equal to the
ratio of the number of units of the resource assigned to his group, to the number of
{\em candidates\/} within his group.

In the policing setting, these models can be viewed as two extremes
of police targeting ability for an intervention like \emph{stop-and-frisk}.
In the random model, police are viewed as stopping people uniformly at random.
In the precision model, police have the omniscient ability to identify individuals
with contraband, and stop only them. Of course, reality lies somewhere in between.

These different discovery models have different implications for fairness. In the random 
model, fairness constrains resources to be distributed
in amounts proportional to group sizes, regardless of the distribution of candidates, 
and so is uninteresting from a learning perspective. On the other hand, the precision 
model yields an interesting fairness-constrained learning problem when the distribution 
of the number of candidates in each group must be learned via observation, and 
what counts as a `fair' allocation depends greatly on these distributions. 

We study learning in a censored feedback setting: each round, the algorithm 
can choose a feasible deployment of resources across groups.
Then the number of candidates for the current round in each group is drawn
from a fixed, but unknown group-dependent distribution (which might be not 
be independent from the distributions in other groups).
The algorithm does not observe the number of candidates present in each group, but
only the number of candidates that received the resource.
In the policing setting, this corresponds to the algorithm being able to observe the number of
arrests, but not the actual number of crimes in each of the districts.
Thus, the extent to which the algorithm can learn about the distribution in a particular group
is limited by the number of resources it deploys there. The goal of the algorithm is to converge to an
optimal fairness-constrained allocation, where here both the objective value of the solution, 
and the constraints imposed on it,
depend on the unknown distributions.

One trivial solution to the learning problem is to sequentially deploy \emph{all} of
one's resources to each group in turn for a sufficient amount of time to accurately learn
the candidate distributions. This would reduce the learning problem to an offline constrained
optimization problem, which we show can be efficiently solved by a greedy algorithm.
But this algorithm is unreasonable:
it has a large exploration phase in which it uses nonsensical deployments, vastly overallocating to some groups
and underallocating to others. A much more realistic, natural approach is a greedy-style learning algorithm,
which at each round simply uses its current best-guess estimate for the distribution in each group
and deploys an optimal fairness-constrained allocation according to these estimates.
Unfortunately, as we show, if one makes no assumptions on the underlying distributions, any algorithm
that has a guarantee of converging to a fair allocation must behave
like the trivial algorithm, deploying vast numbers of resources to each group in turn.

This impossibility result motivates us to consider the learning problem in which 
the unknown distributions are from a known parametric family.  The natural 
greedy algorithm uses an optimal fair deployment at each round given the 
maximum likelihood estimates of candidate distributions given its (censored) 
observations so far; for concreteness, we analyze this algorithm in case of the 
Poisson distribution, and show that it converges to an optimal fair allocation, 
but our analysis generalizes for any single-parameter Lipschitz-continuous family of distributions.

Finally, we conduct an empirical evaluation of our algorithm on the {\em Philadelphia Crime Incidents\/} dataset, 
which records all crimes reported to the Philadelphia Police Department's INCT system between 2006 and 2016. 
We verify that the crime distributions in
each district are in fact well-approximated by Poisson distributions, and that
our algorithm converges quickly to an optimal fair allocation
(as measured according to the empirical crime distributions in the dataset).
We also systematically evaluate the \emph{Price of Fairness}, and plot the Pareto
curves that trade off the number of crimes caught versus the slack allowed in our fairness constraint,
for different sizes of  police force, on this dataset. For the random discovery model, we prove worst-case 
bounds on the Price of Fairness.
\subsection{Further Related Work}
\label{sec:related}

Our precision discovery model is inspired by and has technical connections to~\citet{GanchevKNV09}, 
which models the \emph{dark pool problem} from quantitative finance, in which a trader wishes to 
execute a specified number of trades across a set of exchanges of unknown but independently 
distributed liquidity. 
In \citet{GanchevKNV09}, the authors design an optimal allocation algorithm under the censored 
feedback of the precision model. 
It is straightforward to map their setting onto ours, but they assume independence between 
different exchanges, while  the candidate distributions 
in our setting need not be independent. Regardless, we show that their allocation algorithm 
can be used to compute an optimal allocation (ignoring fairness) even when the independence 
assumption is relaxed (see Remark \ref{rmk:correlation}).
Later,~\citet{AgarwalBD10} extend the dark pool problem to an adversarial (rather than distributional) 
setting. This is quite closely related to the work of
\citet{EnsignFNSV182} who also consider the precision model (under a different name) in an adversarial 
predictive policing setting. They provide no-regret algorithms for this setting by reducing the problem to 
learning in a partial monitoring environment.
Since their setting is equivalent to that of~\citet{AgarwalBD10}, the algorithms in~\citet{AgarwalBD10} can 
be directly applied to the problem studied by \citet{EnsignFNSV182}.

Our desire to study the natural greedy algorithm rather than an algorithm which uses ``unreasonable'' 
allocations during an exploration phase is an instance of a general concern about exploration in 
fairness-related problems \cite{explorewhitepaper}. Recent works have studied the performance of 
greedy algorithms in different settings for this reason \cite{bastani,KannanMRBW18,RaghavanSVW18}.

Lastly, the term \textit{fair allocation} appears in the \emph{fair division} literature 
(see e.g.~\cite{Procaccia13} for a survey), but that body of work is technically quite distinct from the 
problem we study here.
\section{Setting}
\label{sec:model}

We study an \emph{allocator} who has $\V$ units of a resource and is tasked with distributing them
across a population partitioned into $\R$ groups.  Each group is divided into \textit{candidates}, who are the
individuals the allocator would like to receive the resource, and \textit{non-candidates}, who are the
remaining individuals. We let $\m_i$ denote the total
number of individuals in group $i$. The number of candidates $\ci$ in group $i$ is a random variable
drawn from a fixed but unknown 
distribution $\C_i$ called the \emph{(marginal) candidate distribution}.
We do not make any assumptions about the relationship between the candidate distributions
across different groups and in particular these distributions need not be independent. 
We use $M$ to denote the total size of all groups
(i.e., $M = \Sigma_{i \in [\R]} \m_i$).
 An allocation $\v = (v_1, \ldots, v_\R)$ is a partitioning
of these $\V$ units, where $\vi \in \{0, \ldots, \V\}$ denotes the units of resources allocated to
group $i$. Every allocation is bound by a \emph{feasibility} constraint which requires that
$\Sigma_{i \in [\R]} \vi\leq \V$.

A \textit{discovery model} is a (possibly randomized) function $\disc(\vi, \ci)$
mapping the number of units $\vi$ allocated to group $i$ and the number of
candidates $\ci$ in group $i$ to the number of candidates discovered in group $i$. 
In the learning setting, upon fixing an allocation $\v$, the learner will get to observe (a realization of) $\disc(\vi,\ci)$ 
for the realized value of $\ci$ for each group $i$.  Fixing an allocation $\v$,
a discovery model $\disc(\cdot)$ and candidate distributions for all groups $\C=\{\C_i:i\in[\R]\}$, we define
the total expected number of discovered candidates, $\util(\v, \disc(\cdot), \C),$  as
\begin{equation} 
\label{eq:util}
\util\left(\v, \disc(\cdot), \C\right) = \sum_{i \in [\R]}\E_{\ci\sim\C_i}\big[\disc(\vi,\ci)\big], \end{equation}
where the expectation is taken over $\C_i$ and any randomization in the discovery model $\disc(\cdot)$.
When the discovery model and the candidate distributions are fixed,
we will simply write $\util(\v)$ for brevity. We also use 
the total expected number of discovered candidates and \emph{(expected) utility} exchangeably. 
We refer to an allocation that maximizes the expected
number of discovered candidates over all feasible allocations
as an \emph{optimal allocation} and denote it by $\vopt^*$.
\subsection{Allocative Fairness}
\label{sec:fairness}
For the purposes of this paper, we say that an allocation is \emph{fair}
if it satisfies \emph{approximate equality of candidate discovery probability}
across groups. We call this \emph{discovery probability} for brevity. This formalizes the intuition
that it is unfair if candidates in one
group have an inherently higher probability of receiving the resource than candidates in another.
Formally, we define our notion of \emph{allocative fairness} as follows.
\begin{definition}
\label{def:fair}
Fix a discovery model $\disc(\cdot)$ and the candidate distributions $\C$.
For an allocation $\v$, let
$$\f_i\left(\vi, \disc(\cdot), \C_i\right)  =  \E_{\ci\sim\C_i}\left[\frac{\disc\left(\vi,\ci\right)}{\ci}\right],$$
denote the expected probability that a random candidate from group $i$ receives a unit of the resource 
at allocation $\v$ (i.e. the discovery probability in group $i$). Then for any $\alpha \in [0,1]$,  $\v$
is $\alpha$-fair if
\[\Big|\f_i\left(\vi,\disc(\cdot), \C_i\right)-\f_j\left(\vj,\disc(\cdot), \C_j\right)\Big|\leq \alpha,\]
for all pairs of groups $i$ and $j$.
\end{definition}
When it is clear from the context, for brevity, we write $\f_i({\vi})$ for the
discovery probability in group $i$.
We emphasize that this definition (1) depends crucially on the chosen discovery model, 
and (2) requires nothing about the treatment of non-candidates. We think of this as a \emph{minimal} 
definition of fairness, in that one might want to further constrain the treatment of non-candidates --- 
but we do not consider that extension. 

Since discovery probabilities $f_i(\vi)$ and $f_j(\vj)$ are
in $\left[0,1\right]$, the absolute value of their difference is in $\left[0,1\right]$.
By setting $\alpha=1$ we
impose no fairness constraints whatsoever on the allocations, 
and by setting $\alpha=0$
we require \emph{exact} fairness.

We refer to an allocation $\v$ that maximizes $\util(\v)$
subject to $\alpha$-fairness and the feasibility constraint as
an \emph{optimal $\alpha$-fair allocation} and denote it by $\vopt^{\alpha}$.
In general, $\util(\vopt^{\alpha})$ is a non-increasing quantity in 
$\alpha$, since as $\alpha$ diminishes, the utility maximization problem 
becomes more constrained.\footnote{In Appendix~\ref{sec:feasible-in-exp}, 
we show how to compute $\vopt^{\alpha}$ 
for any \emph{arbitrary} but \emph{known} candidate distributions 
$\C$ and known discovery model $\disc(\cdot)$ in a relaxation where the 
feasibility constraint is satisfied in expectation.}

\begin{rmk}
\label{rmk:correlation}
We note that both the utility and discovery probabilities can be written solely 
in terms of the \emph{marginal} candidate distributions in each of the groups, even
when these distributions are not independent.
This is because we have (implicitly) assumed that the number of candidates
discovered in a group depends only on the number of candidates in the group
and the allocation to that group, regardless of the allocations to and the number of 
candidates in other groups. This assumption together with the linearity 
of expectation allows us to write the expected utility as in the right hand side of  
Equation~\ref{eq:util}.
\end{rmk}
\section{The Precision Discovery Model}
\label{sec:min-model}
We begin by describing the \textit{precision model} of discovery. 
Allocating  $\vi$ units to group $i$ in the precision model
results in the discovery of $\disc(\ci,\vi)\triangleq\min(\ci,\vi)$ candidates. This models the 
ability to perfectly discover and reach candidates in a group with resources 
deployed to that group, limited only by the number of deployed resources and 
the number of candidates present. 

The precision model results in  \emph{censored} observations that have a particularly intuitive form. 
Recall that in general, a learning algorithm at each round gets to choose an allocation $\v$ 
and then observes  $\disc(\vi, \ci)$ for each group $i$.
In the precision model, this results in the following kind of observation: when $\vi$ is
 larger than $\ci$, the allocator learns the number of candidates $\ci$ present on that day 
  exactly. We refer to this kind of feedback as an \emph{uncensored observation}. 
 But when $\vi$ is smaller than $\ci$, all the allocator learns is that the number of 
 candidates is \emph{at least} $\vi$. We refer to this kind of feedback as a 
 \textit{censored observation}.

The rest of this section is organized as follows.
In Sections~\ref{sec:min-model-opt}
and~\ref{sec:min-model-opt-fair} we characterize 
optimal and optimal fair allocations
for the precision model when the candidate distributions are known.
In Section~\ref{sec:min-model-learning} we focus on 
learning an optimal fair allocation
when these distributions are unknown. We show that any learning
algorithm that is guaranteed to find a fair allocation in the 
\emph{worst case} over candidate distributions must have the 
undesirable property that at some point, it must allocate a vast number
of its resources to each group individually. To bypass this hurdle,
in Section~\ref{sec:mle-poisson} we show that when the candidate 
distributions have a
parametric form, a natural greedy algorithm which always uses an 
optimal fair allocation for the current maximum likelihood estimates 
of the candidate distributions converges to an optimal fair allocation. 
\subsection{Optimal Allocation}
\label{sec:min-model-opt}
We first describe how an optimal allocation (absent fairness constraints) can be 
computed efficiently when the candidate 
distributions $\C_i$ are known.
In \citet{GanchevKNV09}, the authors provide an algorithm for computing an
optimal allocation when the distributions over the number of
 shares present in each dark pool are known and the trader wishes to
maximize the expected number of traded shares. 
While they assume that 
the distributions of shares across different dark pools are independent, 
our formulation does not require
this assumption of independence.
Regardless, we can use the same algorithm as in~\citet{GanchevKNV09}
to compute an optimal allocation in our setting; this is because, as stated in 
Remark~\ref{rmk:correlation}, the utility in both settings can be written solely in terms of the 
(marginal) candidate distributions even when 
the candidate distributions are not independent across different groups. 
Here, we present the high level ideas
of their algorithm in the language of our model, and provide full details
for completeness
in Appendix~\ref{sec:dark-pool}.

Let $\T_i(c) = \Pr_{\ci\sim \C_i}[\ci \geq c]$
denote the probability that there are at least $c$ candidates
in group $i$. We refer to $\T_i(c)$ as the \textit{tail probability of
$\C_i$ at $c$}. Recall that the value of the \emph{cumulative distribution function} 
(CDF) of $\C_i$ at $c$
is defined to be $$\cdf_i(c) = \sum_{c'\leq c}\Pr_{\ci\sim \C_i}\left[\ci = c'\right].$$ So
$\T_i(c)$ can be written in terms of CDF values as $\T_i(c) = 1-\cdf_i(c-1)$.

First, observe that the expected total
number of candidates discovered by an allocation in the precision model
can be written in terms of the tail probabilities of the candidate distributions i.e.
\begin{equation*}
\label{eq:obj}
\util(\v, \disc(\cdot), \C) = \sum_{i\in[\R]}\E_{\ci\sim\C_i}\left[\min\left(\vi,\ci\right) \right] =
\sum_{i\in[\R]}\sum\limits_{c=1}^{\vi}\T_i(c).
\end{equation*}
Since the objective function is concave (as $\T_i(c)$ is a non-increasing function in $c$ for all $i$), a greedy 
algorithm which iteratively allocates the next
unit of the resource to a group in
\begin{equation*}
\label{eq:greedy}
  \argmax\limits_{i\in[\R]} \left(\T_i\left(\vit+1\right)-\T_i\left(\vit\right)\right),
  \end{equation*}
 where $\vit$ is the current allocation to group $i$ in the $t$\textsuperscript{th} 
 round achieves an optimal allocation. 
\subsection{Optimal Fair Allocation}
\label{sec:min-model-opt-fair}

We next show how to compute an optimal $\alpha$-fair allocation in the precision model
when the candidate distributions are known and do not need to be learned.

To build intuition for how the algorithm works, imagine that the group $i$ has the highest 
discovery probability in  $\vopt^{\alpha}$, and the allocation  $w^{\alpha}_i$ to that group 
is somehow known to the algorithm ahead of time. The constraint of $\alpha$-fairness then 
implies that the discovery probability for each other 
group $j$ in $\vopt^{\alpha}$ must satisfy $\f_j\in[\f_i-\alpha, \f_i]$. This in turn implies upper 
and lower bounds on the feasible allocations $w^{\alpha}_j$ to group $j$. The algorithm is 
then simply a constrained greedy algorithm: subject to these implied constraints, it iteratively 
allocates units so as to maximize their marginal probability of reaching another candidate. 
Since the group $i$ maximizing the discovery probability in  $\vopt^{\alpha}$ and the 
corresponding allocation  $w^{\alpha}_i$ are not known ahead of time, the algorithm simply 
iterates through each possible choice of $i$. 

\begin{algorithm}[ht!]
\begin{algorithmic}
\INPUT $\alpha$, $\C$ and $\V$.
\OUTPUT An optimal $\alpha$-fair allocation $\vopt^{\alpha}$.
\State $\vopt^{\alpha} \gets \vec{0}$. \Comment{Initialize the output.}
\State $\util_{\max}\gets 0$. \Comment{Keep track of the utility of the output.}
\For{$i \in [\R]$}\Comment{Guess for group with the highest probability of discovery.}
\State $\v\gets \vec{0}$.
\For{$\vi \in\{0, \ldots \V\}$}\Comment{Guess for the allocation to that group.}
\State Set $\vi$ in $\v$ and compute $\f_i(\vi)$.
\State $ub_i \gets \vi$. \Comment{Upper bound on allocation to group $i$.}
\State $lb_i \gets \vi$. \Comment{Lower bound on allocation to group $i$.}
\For{$j \ne i, j \in [\R]$} \Comment{Upper and lower bounds for other groups.}
\State Update $lb_j$ and $ub_j$ using $\f_i(\vi)$, $\alpha$ and $\C_j$.
\State $\vj \gets lb_j$. \Comment{Assign the lower bound allocation to group $j$.}
\EndFor
\If{$\Sigma_{i\in [\R]} v_i > \V$}
\State \textbf{continue.} \Comment{Allocation is not feasible.}
\EndIf
\State $\V_r = \V-\Sigma_{i\in [\R]} v_i$
\For{$t = 1, \ldots, \V_r$}\Comment{Allocate the remaining resources greedily while obeying fairness.}
\State $j^* \in \argmax\limits_{j\in[\R]} \left(\T_j(\vj+1)-\T_j(\vj)\right)$ s.t. $\vj< ub_j$.
\State $v_{j^*}\gets v_{j^*}+1$.
\EndFor
\State $\util(\v) = \Sigma_{i\in[\R]}\Sigma_{c=1}^{\vi}\T_i(c)$. \Comment{Compute the utility of $\v$.}
\If{$\util(\v) > \util_{\max}$}\Comment{Update the best $\alpha$-fair allocation found so far.}
\State $\util_{\max}\gets \util(\v).$
\State $\vopt^\alpha\gets \v$.
\EndIf
\EndFor
\EndFor
\State \Return $\vopt^{\alpha}$.
\end{algorithmic}
\caption{Computing an optimal fair allocation in the precision model}
\label{alg:opt-fair-dark-pool}
\end{algorithm}
\normalsize

Pseudocode is given in Algorithm~\ref{alg:opt-fair-dark-pool}. We prove that
Algorithm~\ref{alg:opt-fair-dark-pool} returns an optimal $\alpha$-fair allocation
in Theorem~\ref{thm:fair-opt-min}.
\ifproofs
\else
We  defer the proof of Theorem~\ref{thm:fair-opt-min} and all the other omitted proofs in the section to
Appendix~\ref{sec:dark-pool}.
\fi
\begin{thm}
\label{thm:fair-opt-min}
Algorithm~\ref{alg:opt-fair-dark-pool} computes an optimal $\alpha$-fair
allocation for the precision model in time $O(\R\V(\R\V+M))$.
\end{thm}
\ifproofs
\begin{proof}
Fix an optimal $\alpha$-fair allocation $\vopt^{\alpha}$. 
In $\vopt^{\alpha}$, some group $i$ has the highest $\f_i$ and
receives allocation $w^\alpha_i$.
Suppose we know $i$ and $w^{\alpha}_i$ (we relax this assumption
at the end of the proof). Using the knowledge of $\C_i$
we can compute $f_i(w^\alpha_i)$. This implies that $\f_j \in [\f_i-\alpha, \f_i]$ for every other
group $j$, which in turn can be
used to derive the set of all possible allowable allocations $w^\alpha_j$
which do not violate $\alpha$-fairness.
	
We claim that if we initialize the allocation $w^\alpha_j$ to be the lower
bound of the interval corresponding to group $j$,
then greedily assign the surplus units with the added restriction
that $w^\alpha_j$ is always inside of its respective interval,
we achieve an optimal $\alpha$-fair allocation.
	
Since we assume we know $w^\alpha_i$ to be the allocation to group $i$ in an optimal
fair allocation, this allocation must be achievable by picking some value from each 
of the intervals,
thus initializing the allocation to the lower bound of each interval certainly cannot assign
more than $\V$ units in total.  By the same argument as for the unconstrained
greedy algorithm,
since the objective function is concave (recall that the tail probabilities
are non-increasing) and increasing in each argument $w^\alpha_j$, a greedy search
over this feasible region finds the desired allocation.
	
The algorithm does not know \textit{a priori} the group $i$
which has the maximum $f_i$ or $w^\alpha_i$, so
it must search over these options.
There are $\R$ guesses for group $i$ and $\V+1$ guesses for the
allocation to the group.
So there are a total of $\R(\V+1)$ guesses
that needs to be considered. For each guess, it takes
$O(M)$ to compute the upper and lower bounds on the allocation to each of
the groups and $O(\R\V)$ to run the greedy algorithm. 
So the running time of Algorithm~\ref{alg:opt-fair-dark-pool} is $O(\R\V(\R\V+M))$.
\end{proof}
\else
\fi 
\subsection{Learning Fair Allocations Generally Requires Brute-Force Exploration}
\label{sec:min-model-learning}

In Sections~\ref{sec:min-model-opt}~and~\ref{sec:min-model-opt-fair}
we assumed the candidate
distributions were known. When the candidate distributions are unknown, learning
algorithms intending to converge to optimal $\alpha$-fair allocations must learn a
sufficient amount about the distributions in question to certify the fairness of the
allocation they finally output. Because learners must deal with feedback in the
censored observation model, this places constraints on how they can proceed.
Unfortunately, as we show in this section, if candidate distributions are allowed
to be worst-case, this will force a learner to engage in what we call
``brute-force exploration'' --- the iterative deployment of a large fraction
of the resources to each subgroup in turn.
This is formalized in Theorem~\ref{thm:imp}.
\begin{thm}
\label{thm:imp}
Define $\m^* = \max_{i\in[\R]} \m_i$ to be the size of the largest group
and assume $\m_i>6$ for all $i$ and $\R \geq 2$. Let $\alpha\in [0,1/(2\m^*))$, 
$\delta\in(0,1/2)$, and $\alloc$ be
any learning algorithm for the precision model
which runs for a finite number of rounds and outputs an allocation.
Suppose that there is some group $i$ for which $\alloc$ has not allocated
at least $m_i/2$ units for at least $k\ln(1/\delta)/(\alpha\m_i)$ rounds upon termination, 
where $k$ is an absolute constant.		
Then there exists a candidate distribution such that, with probability at least $\delta$,
$\alloc$ outputs an allocation that is not $\alpha$-fair.
\end{thm}

\ifproofs
\begin{proof}	
Let $i$ denote the group in which $\alloc$ has not allocated at least $m_i/2$ units for 
at least $k\ln(1/\delta)/(\alpha\m_i)$ rounds upon its termination. We fix an arbitrary allocation $\v$ and  design two candidate
distributions for group $i$ such that the discovery probabilities given $\vi$ computed under the two different 
distributions are at least 2$\alpha$-apart.\footnote{We assume $\v$ sends at least one unit to group $i$,
otherwise it would be easy to construct an example where the algorithm allocating according to $\v$ is unfair.}
Any algorithm guaranteeing $\alpha$-fairness must distinguish between these two distributions with 
high probability, or $\v$ could have higher unfairness than $\alpha$. We then show that to distinguish 
between these two candidate distributions,
with probability of at least $1-\alpha$, any algorithm is \emph{required} to send $\m_i/2$
units to group $i$ for at least $k\ln(1/\delta)/(\alpha\m_i)$ rounds.

Consider two candidate distributions $\C_i$ and $\C'_i$ for group $i$.
We use the shorthand $p_i(c)=\Pr_{\ci\sim \C_i}[\ci = c]$ and similarly for
$p'_i(c)$.
Let $c^*=\m_i/2-2$. We require only that $\C_i$ and $\C'_i$ satisfy the following conditions.
\begin{enumerate}
\item $p_i(c)=p'_i(c)$ for all $c'\leq c^*$.
\item $\Sigma_{c\leq c^*}p_i(c) = \Sigma_{c\leq c^*}p'_i(c) =1-2\alpha\m_i.$
\item  $p_i(c^*+1) = 2\alpha\m_i$ and $p_i(c) = 0$  for all $c \in \{c^*+1, \ldots, \m_i\}$.
\item $p'_i(c) = 0$  for all $c\in\{c^*+1, \ldots, \m_i-1\}$ and $p'_i(\m_i)=2\alpha\m_i$.
\end{enumerate}
In other words, any two distributions that are the same up to $c^*$, have a CDF value of $2\alpha \m_i$ at $c^*$, 
and differ in where in the tail they assign the remaining mass, will serve our purposes.

Let $f_i(\vi)$ and $f'_i(\vi)$ denote the discovery probability given allocation $\v$
which assigns $\vi$ units to group $i$ for candidate distributions $\C_i$ and $\C'_i$, respectively.
Then
\[
\left|f_i(\vi)-f'_i(\vi)\right| = 2\alpha\m_i\left|\frac{\vi}{c^*+1}-\frac{\vi}{\m_i}\right|.
\]
This difference is
minimized at $\vi=1$, in which case
$$
\left|\frac{\vi}{c^*+1}-\frac{\vi}{\m_i}\right| > \frac{2}{\m_i}-\frac{1}{\m_i} = \frac{1}{\m_i}.
$$
Hence, for any allocation $\v$,
$\left|f_i(\vi)-f'_i(\vi)\right| > 2\alpha.$

Finally, because $\C_i$ and $\C_i'$ do not differ on any potential observation less than $c^*$, distinguishing 
between the two candidate distributions requires observing at least one uncensored observation of $c^*$ or 
higher. Under the precision model, this requires sending at least $\m_i/2$ units
to group $i$. However, conditioning on sending at least $\m_i/2$
units, the probability of observing an uncensored observation is at most $2\alpha \m_i$. Hence, to distinguish 
between $\C_i$ and $\C_i'$ (and thus to guarantee that an allocation $\v$ is $\alpha$-fair) with probability of 
at least $1-\delta$,  a learning algorithm must allocate $\m_i/2$ units for $k\ln(1/\delta)/(\alpha\m_i)$ rounds
to group $i$.
\end{proof}
\else
\begin{proof}[Sketch of the Proof]
Let $i$ denote a group in which $\alloc$ has not allocated at least $m_i/2$ units for
at least $k\ln(1/\delta)/(\alpha\m_i)$ rounds upon its termination and let  $\v$
denote an arbitrary allocation.
We will design two candidate distributions  for group $i$ which have true discovery probabilities
that are at least $2\alpha$ apart given $\vi$, but which are indistinguishable given the 
observations of the algorithm with probability at least $\delta$.  If the $\alloc$ cannot 
distinguish between $\C_i$ and $\C_i'$, it cannot distinguish between $\f_i$ and $\f_i'$, 
and thus cannot guarantee whether group $i$'s discovery probability is indeed within $\alpha$ 
of every other group's discovery probability. 

To design these candidate distributions, consider distributions $\C_i$
and $\C_i'$ which satisfy the following four conditions.
\begin{enumerate}
	\item $\C_i$ and $\C_i'$ agree on all values less than $m_i/2-2$.
	\item The total mass of both distributions below $m_i/2-2$ is $1-2\alpha m_i$.
	\item The remaining $2\alpha m_i$ mass of $\C_i$ is on the value $m_i/2-1$.
	\item The remaining $2\alpha m_i$ mass of $\C_i'$ is on the value $m_i$.
\end{enumerate}

Distinguishing between $\C_i$ and $\C'_i$ requires at least one uncensored observation
beyond $\m_i/2-2$.  However, conditioned on allocating at least $\m_i/2$ units,
the probability of observing an uncensored observation is at most $2\alpha m_i$.
So to distinguish between $\C_i$ and $\C_i'$ with confidence $1-\delta$, and therefore to guarantee an 
$\alpha$-fair allocation, a learning algorithm must allocate at least $\m_i/2$ units
to group $i$ for $k\ln(1/\delta)/(\alpha \m_i)$ rounds.
\end{proof}
\fi

Recall that we used $m^*$ to denote the size of the largest group.
When $m^*>2\V$, then Theorem~\ref{thm:imp} implies that no algorithm
can guarantee $\alpha$-fairness for sufficiently small $\alpha$.
Moreover, even when $m^*\leq 2\V$,
Theorem~\ref{thm:imp} shows that in general, if we want algorithms that have
provable guarantees for \emph{arbitrary} candidate distributions, it is impossible
to avoid something akin to brute-force search (recall that there is a trivial algorithm
which simply allocates \emph{all} resources to each group in turn, for sufficiently
many rounds to approximately learn the CDF of the candidate distribution, and
then solves the offline problem). In the next section, we circumvent this by
giving an algorithm with provable guarantees, assuming that the candidate distributions
have a known parametric form.
\newcommand{\Tmax}{T_{\max}}

\subsection{Poisson Distributions and Convergence of the MLE}
\label{sec:mle-poisson}

In this section, we assume that all the candidate distributions have a particular and known \textit{parametric form}
but that the parameters of the these distributions are not known to the allocator.
Concretely, we assume that the candidate distribution for each group is Poisson\footnote{To match
our model, we would technically need to assume
a \emph{truncated} Poisson distribution to satisfy the bounded support condition. However,
the distinction will not be important for the analysis, and so to minimize technical overhead,
we perform the analysis assuming an untruncated Poisson.}  (denoted by $\C(\lambda)$)
and write $\lambvec^*=(\lambda^*_1, \dots, \lambda^*_\R)$ for the true underlying parameters of the
candidate distributions; this choice appears justified, at least in the predictive policing application, as the
candidate distributions in the Philadelphia Crime Incidents  dataset are
well-approximated by Poisson distributions (see Section \ref{sec:exp} for further discussion).
This assumption allows an algorithm to learn the tails of these distributions without needing
to rely on brute-force search, thus circumventing the limitation given in Theorem \ref{thm:imp}.
Indeed, we show that (a small variant of) the natural greedy algorithm incorporating these distributional assumptions
converges to an optimal fair allocation.

For simplicity, we assume a parametric form on the marginal candidate distribution in each of the groups. 
We could have equivalently assumed that the candidates across groups are drawn from a multivariate 
Poisson distribution to highlight the (potential) correlation between 
candidates distributions. However, since for a given multivariate Poisson distribution the marginal distribution
on each group is itself a Poisson distribution~\cite{InouyeYAR2017}, we made our parametric assumption directly on
these marginal distributions.

At a high level, in each round, our algorithm uses
Algorithm \ref{alg:opt-fair-dark-pool} to calculate an optimal fair allocation
with respect to the
current maximum likelihood estimates of the group distributions; then, it uses
the new observations it obtains from this allocation to refine these estimates for the next round.
This is summarized in Algorithm \ref{alg:learnplayfair}. The algorithm differs from this pure greedy strategy 
in one respect, to overcome the following subtlety: there is a possibility that Algorithm
\ref{alg:opt-fair-dark-pool}, when operating on a preliminary estimate for the candidate distributions, 
will suggest sending zero units to
some group, even when the optimal allocation for the true distributions sends some units to every 
group. Such a deployment would result in the algorithm receiving no feedback for the zero-allocated group that round. 
If this suggestion is followed and a lack of feedback
is allowed to persist
indefinitely, the algorithm's parameter estimate for the zero-allocated group will also stop updating --- potentially at an incorrect value. In
order to avoid this problem and continue making progress in learning, our algorithm 
chooses another allocation in this case. As we show, any allocation that allocates positive resources to
all groups will suffice; in particular, our algorithm makes the natural choice of simply repeating the allocation from the previous round.

\begin{algorithm}[ht!]
\begin{algorithmic}
\INPUT $\alpha$, $\V$ and $T$ (total number of rounds).
\OUTPUT An allocation $\v^{T+1}$ and estimates to parameters $\{\lambda_i^T\}$.
\State $\v^1 \gets (\lfloor(\V/\R)\rfloor, \ldots,\lfloor(\V/\R)\rfloor)$.\Comment{Allocate uniformly.}
\For{rounds $t=1, \dots,T$}
\If{$\exists i \text{ such that }\vit==0$}\Comment{Check whether every group is allocated a resource.}
\State $\v^t \gets \v^{t-1}$.
\EndIf
\State Observe $\caughtit = \min\{\vit, \cit\}$ for each group.
\For{$i=1, \ldots, \R$}
\State Update history $\hvec_i^{t+1}$ with $\caughtit$ and $\vit$.
\State $\hat{\lambda}_i^t \gets \argmax_{\lambda \in [\lambda_{\min},\lambda_{\max}]} \hat{\like}(\hvec_i^{t+1}, \lambda)$. \Comment{Solve the maximum likelihood estimation problem.}
\EndFor
\State $\v^{t+1} \gets \text{Algorithm}~\ref{alg:opt-fair-dark-pool}(\alpha, \{\C(\hat{\lambda}_i^t)\}, \V)$. \Comment{Compute an allocation to be deployed in the next round.}
\EndFor
\State \Return $\v^{T+1}$ and $\{\lambda_i^T\}$.
\end{algorithmic}
\caption{Learning an optimal fair allocation}
\label{alg:learnplayfair}
\end{algorithm}

Notice that Algorithm~\ref{alg:learnplayfair} chooses an allocation at every round which is fair
with respect to its estimates of the parameters of the candidate distributions; hence, asymptotic
convergence of its output to an \emph{optimal $\alpha$-fair}
allocation follows directly from the convergence of the estimates to true parameters.
However, we seek a stronger, \emph{finite sample} guarantee, as stated in
Theorem \ref{thm:convgoptfair}.

\begin{thm}
\label{thm:convgoptfair}
Let $\epsilon ,\delta >0$.
Suppose that the candidate distributions are Poisson distributions with
unknown parameters in the vector $\lambvec^*$, where $\lambvec^*$ lies in the known interval $[\lambda_{\min}, \lambda_{\max}]^\R$.
Suppose we run Algorithm~\ref{alg:learnplayfair} for $t> \tilde\bigo(\ln(\R/\delta)/(\eta(\epsilon))^2) \triangleq \Tmax$ rounds, where $\eta(\cdot)$
is some distribution specific function\footnote{See Corollary \ref{cor:closeness} in Appendix~\ref{sec:dark-pool} for the relationship between $\eta$ and $\epsilon$.
Also  $\tilde\bigo$ hides poly-logarithmic terms in $1/\eta(\epsilon).$}
to get an allocation $\hat{\v}$ and estimated parameters $\hat{\lambda}_i$ for all groups $i$. Then
with probability at least $1-\delta$
\begin{enumerate}
\item For all $i$ in $[\R]$, $|\hat{\lambda}_i-\lambda^*_i|\le \epsilon$.
\item Let $D = \max_{i \in [\R]} D_{TV}(\C(\lambda^*_i), \C(\hat{\lambda}_i))$ where
$D_{TV}$ denotes the total variation distance between two distributions.
Then $\hat{\v}$
\begin{itemize}
\item is $(\alpha+4D)$-fair.
\item has utility at most $4D\R\V$ smaller than the utility of an optimal $(\alpha-4D)$-fair
allocation i.e. $\util(\hat{\v}) \geq \util(\vopt^{\alpha-4D})-4D\R\V$.
\end{itemize}
\end{enumerate}
\end{thm}

\begin{rmk}
Theorem~\ref{thm:convgoptfair} implies that in the limit, the allocation from Algorithm~\ref{alg:learnplayfair}
will converge to an optimal $\alpha$-fair allocation. As $t \to \infty$,
$\hat{\lambda}_i \overset{p}{\to} \lambda^*_i$ for all $i$, meaning $D \to 0$ and more importantly, $\hat{\v}$
will be $\alpha$-fair and optimal.
\end{rmk}

The rest of this section is dedicated to the proof of Theorem~\ref{thm:convgoptfair}.
First, we introduce notation.
Since we assumed the candidate distribution for each group is Poisson,
the \emph{probability mass function} (PMF) and the CDF of the candidate
distribution $\C_i$ for group $i$ can be written as
\begin{align*}
\Pr_{\ci\sim \C(\lambda^*_i)}\left[\ci = c; \lambda^*_i\right] = \frac{{\lambda^*_i}^c e^{-\lambda^*_i}}{c!}
\text{ and } F(c;\lambda^*_i)=\Pr_{\ci\sim\C(\lambda^*_i)}\left[\ci \leq c;\lambda^*_i\right]
= e^{-\lambda^*_i}\sum_{x=0}^{c} \frac{{\lambda^*_i}^x}{x!}.
\end{align*}
Given an allocation of $\vi$ units of the resource to group $i$ we use $\caughti$ to denote the
(possibly censored) observation received by Algorithm~\ref{alg:learnplayfair}.
So while the candidates in group $i$ are generated according to $\C(\lambda^*_i)$,
the observations of Algorithm~\ref{alg:learnplayfair} follow a censored Poisson distribution
which we abbreviate by $\C_o(\lambda^*_i, \vi)$. We can write the PMF of this distribution as
\begin{align*}
\Pr_{\caughti\sim\C_o(\lambda^*_i, \vi)}[\caughti =o ; \lambda^*_i, \vi] =
\begin{cases} \frac{{\lambda^*_i}^o e^{-\lambda^*_i}}{o!}, & o < \vi, \\ 1-F(\vi-1; \lambda^*_i), & o = \vi,\end{cases}
\end{align*}
where $F(\vi-1; \lambda^*_i)$ is the CDF value of $\C(\lambda^*_i)$ at $\vi-1$.

Since Algorithm~\ref{alg:learnplayfair} operates in rounds, we use the superscript $t$ throughout to
denote the round.
For each round $t$, denote the history of the units allocated to group $i$ and
observations received (candidates discovered) in rounds
up to $t$ by $\histit=(\vione, \caughtione, \hdots,  \vitmone,\caughtitmone)$.
We use $\hvec^t = (\hist^t_1, \dots, \hist^t_\R)$ to denote
the history for all groups.
All the probabilities and expectations in this section are over the randomness
of the observations drawn from the censored Poisson distributions unless otherwise noted; 
we suppress related notation for brevity.
Finally, an allocation function $\alloc$ in round $t$ is a mapping from the history
of all groups $\histt$ to the number of units
 to be allocated to each group i.e. $\alloc:\histt\rightarrow \vt$. For convenience,
 we use $\alloc(\histt)_i$ to denote the allocation
 $\vit$ at round $t$.
We are now ready to define the likelihood functions.
\begin{definition}
Let $p(\vit,\caughtit; \lambda) := \Pr [\caughtit; \lambda, \vit]$ denote the (censored)
likelihood of discovering $\caughtit$ candidates
given an allocation $\vit$ to group $i$ assuming the candidate distribution follows
$\C(\lambda)$. We write $\ell(\vit, \caughtit; \lambda)$ as $\log p(\vit, \caughtit; \lambda)$.
So, given any history $\hvec^t$,  the \emph{empirical} log-likelihood function for group $i$ is
\begin{align*}
\hat{\like}_i\left(\hvec^t,\lambda\right) = \frac{1}{t} \sum_{s=1}^{t} \ell\left(\alloc(\hvec^s)_i,\caughtis; \lambda\right).
\end{align*}
The expected log-likelihood function given the history of allocations
but over the randomness of the candidacy distribution can be written as
\begin{align*}
\like^*_i\left(\hvec^t,\lambda\right) = \frac{1}{t}\sum_{s=1}^{t} \E \left[\ell\left(\alloc(\hvec^s)_i,\caughts_i; \lambda\right)\right],
\end{align*}
where the expectation is over the randomness of $\caughts_i$ drawn from  $\C_o(\lambda^*, \alloc(\hvec^s)_i)$.
\end{definition}

\subsection*{Proof of Theorem~\ref{thm:convgoptfair}}
To prove Theorem \ref{thm:convgoptfair}, we first show that \emph{any}
sequence of allocations selected by Algorithm~\ref{alg:learnplayfair}
will eventually recover the true parameters. There are two conceptual difficulties here: the first is that
standard convergence results typically leverage the assumption of \textit{independence},
which does not hold in this case as Algorithm~\ref{alg:learnplayfair} computes \textit{adaptive}
allocations which depend on the allocations in previous rounds; the second is the
censoring of the observations. Despite these difficulties, we give quantifiable rates with which the
estimates converge to the true parameters.
Next, we show that computing an optimal $\alpha$-fair allocation using the estimated parameters
will result in an allocation that is $(\alpha +4D)$-fair with respect to the true candidate
distributions where $D$ denotes the maximum total variation distance between the true and estimated
Poisson distributions across all groups. Finally, we show that this allocation also achieves a utility
that is comparable to the utility of an optimal $(\alpha-4D)$-fair allocation.
We note that while Theorem~\ref{thm:convgoptfair} is only stated for Poisson distributions,
our results can be generalized to
any single parameter Lipschitz-continuous family of distributions (see Remark~\ref{rm:gen}). 
\noindent\paragraph{Closeness of the Estimated Parameters}
Our argument can be stated at a high level as follows: 
for any group $i$ and any history $\hvec^t$, the empirical log-likelihood 
converges to the expected log-likelihood for any sequence of allocations made by 
Algorithm~\ref{alg:learnplayfair} as formalized in Lemma~\ref{thm:uniform}. We then 
show in Lemma~\ref{lem:convergence}
that the closeness of the empirical and expected log-likelihoods implies that the maximizers
of these quantities (corresponding to the estimated and true parameters) will also become close.
Since in our analysis we consider the groups separately, we fix a group $i$ throughout the rest of this
section and drop the subscript $i$ for convenience.

We start by studying the rate of convergence of the empirical log-likelihood 
to the expected log-likelihood.
\begin{lemma}
With probability at least $1-\delta/\R$, for any $t$ and any 
$\histt$ observed by Algorithm~\ref{alg:learnplayfair}
\begin{align*}\sup_{\lambda \in [\lambda_{min}, \lambda_{max}]} \left|\hat{\like}\left(\histt, \lambda\right) - 
\like^*\left(\histt, \lambda\right) \right| \le \bigo\left(\sqrt{\frac{\ln(t\R/\delta)}{t}}\right).
\end{align*}
\label{thm:uniform}
\end{lemma}

The true and 
estimated parameters for each group correspond to the maximizers 
of the expected and empirical log-likelihoods, respectively (see Corollary~\ref{cor:1} in 
Appendix~\ref{sec:dark-pool}). 
We next show that 
closeness of the empirical and expected log-likelihoods implies that
the true and estimated parameters are also close.
\begin{lemma}
\label{lem:convergence}
Let $\hat{\lambda}$ denote the estimate of the Algorithm~\ref{alg:learnplayfair} after 
$T_{\max}$ rounds. Then with probability at least  $1- \delta/\R$, $\ |\hat{\lambda}-\lambda^*| < \epsilon$.
\end{lemma}
\begin{proof}
Since Corollary~\ref{cor:1}(in Appendix~\ref{sec:dark-pool}) gives that $\like(\histt, \lambda)$ has a unique maximizer 
at $\lambda^*$ and Corollary \ref{cor:closeness}(in Appendix~\ref{sec:dark-pool}) gives that there exists some $\eta(\epsilon)$ 
so that for any $\lambda'$ such that 
$|\like^*(\histt, \lambda') - \like^*(\histt, \lambda^*)| < \eta(\epsilon)$, we must have that 
$|\lambda' - \lambda^*| < \epsilon$. We denote $\eta(\epsilon)$ by $\eta$ for brevity. 
We define the empirical maximizer $\hat{\lambda}$ to be the maximizer of $\hat{\like}(\histt, \lambda)$ i.e.
	\begin{align} \label{ineq:likelambdahat}
	\hat{\lambda} \in \argmax_{\lambda \in \left[\lambda_{\min},\lambda_{\max}\right]} \hat{\like}\left(\histt, \lambda\right).
	\end{align}
 Applying Lemma~\ref{thm:uniform} 
 implies that for any $\histt$ with $t> \Tmax$ and $\lambda' \in [\lambda_{\min},\lambda_{\max}]$, with probability at least $1-\delta/\R$,
\begin{align*}
 \left|\hat{\like}(\histt, \lambda') - \like^*(\histt, \lambda')\right| <\frac{\eta}{2}.\end{align*}
In particular, we must have that
\begin{align}\label{ineq:minlambdstar}
\left|\hat{\like}\left(\histt, \lambda^*\right) - \like^*\left(\histt, \lambda^*\right)\right| < \frac{\eta}{2} \qquad \text{and} \qquad 
\left|\hat{\like}\left(\hvec^t,\hat{\lambda}\right) - \like^*\left(\hvec^t,\hat{\lambda}\right)\right| < \frac{\eta}{2}.
\end{align}

Since $\hat{\lambda}$ is a maximizer of Equation \ref{ineq:likelambdahat}, we have that
 \begin{align*}
\hat{\like}\left(\histt, \hat{\lambda}\right) \geq \hat{\like}\left(\histt, \lambda^*\right) 
\geq \like^*\left(\histt, \lambda^*\right) - \frac{\eta}{2},
\end{align*}
where the last inequality is by Equation~\ref{ineq:minlambdstar}.
This implies that
\begin{align*}
\like^*\left(\histt, \hat{\lambda}\right) > \hat{\like} \left(\histt, \hat{\lambda}\right) - \frac{\eta}{2} > 
\like^*\left(\histt, \lambda^*\right) - \frac{\eta}{2} - \frac{\eta}{2} = \like^*\left(\histt, \lambda^*\right) - {\eta},
\end{align*}
where the last inequality is by Equation~\ref{ineq:minlambdstar}.
So $\like^*(\histt, \hat{\lambda}) > \like^*(\histt, \lambda^*)-\eta$, and thus 
Corollary \ref{cor:closeness} 
gives that $|\hat{\lambda}-\lambda^*| < \epsilon$.
\end{proof}
Combining Lemma \ref{lem:convergence}  
with a union bound over all groups 
show that, with probability $1-\delta$,  if Algorithm \ref{alg:learnplayfair} is run for 
$\Tmax$ rounds, then $|\hat{\lambda}_i-\lambda_i^*| < \epsilon$, for all $i$. Note that as $t \to \infty$, 
the maximum total variation distance $D$ between the estimated and the true distribution will converge in probability to 0.
\noindent\paragraph{Fairness of the Allocation}
In this section, we show that the fairness violation 
(i.e. the maximum difference in discovery probabilities over all pairs of groups) is linear in terms of $D$.
Therefore, as the running time of the Algorithm~\ref{alg:learnplayfair} increases and hence, $D \to 0$, 
the fairness violation of $\hat{\v}$ approaches $\alpha$. This is stated formally as follows. 
\begin{lemma}
\label{lem:close-fairness-index}
Let $\hat{\v}$ denote the allocation returned by Algorithm~\ref{alg:learnplayfair}
after $\Tmax$ rounds. Then with probability at least $1-\delta$,
$\left|\f_i(\hat{\vsingle}_i) - \f_j(\hat{\vsingle}_j) \right| \le \alpha + 4D, \forall i,j \in [\R].$
\end{lemma}
\begin{proof}
For any $i, j \in [\R]$ we have that
\begin{align*}
&\left|\f_i(\hat{\vsingle}_i) - \f_j(\hat{\vsingle}_j) \right| = \left|\f_i(\hat{\vsingle}_i, C(\lambda^*_i)) - \f_j(\hat{\vsingle}_j, C(\lambda^*_j)) \right|\\
&\le \left|\f_i(\hat{\vsingle}_i, \C(\lambda^*_i)) - \f_i(\hat{\vsingle}_i, \C(\hat{\lambda}_i)) \right|
+  \left|\f_i(\hat{\vsingle}_i, C(\hat{\lambda}_i)) - \f_j(\hat{\vsingle}_j, \C(\hat{\lambda}_j)) \right| + 
\left|\f_j(\hat{\vsingle}_j, \C(\lambda^*_j)) - \f_j(\hat{\vsingle}_j, \C(\hat{\lambda}_j))\right|\\
&\le  2D_{TV}(\C(\lambda^*_i), \C(\hat{\lambda}_i)) + \alpha + 2D_{TV}(\C(\lambda^*_j), \C(\hat{\lambda}_j))\\
&= \alpha + 4D.
\end{align*}
The first inequality follows from the triangle inequality.
In the second inequality, the second term can be bounded because Algorithm~\ref{alg:opt-fair-dark-pool}
returns an $\alpha$-fair allocation
with respect to its input distribution.
The first and third term in the second inequality can be bounded by Lemma~\ref{lem:disc-prob} 
(in Appendix~\ref{sec:dark-pool}). 
Lemma~\ref{lem:disc-prob} shows that for any fixed allocation the difference between 
the discovery probability with respect to the true and estimated candidate distributions
in group $i$ is proportional to the total variation distance between the true and estimated
distributions. 
\end{proof}

\noindent\paragraph{Utility of the Allocation}
In this section we analyze the utility of the allocation returned by 
Algorithm~\ref{alg:learnplayfair}. Once again, note that as $D \to 0$, which happens as the 
running time of Algorithm \ref{alg:learnplayfair} increases, $\hat{\v}$ will become 
optimal and $\alpha$-fair.
\begin{lemma}
\label{lem:close-util}
Let $\hat{\v}$ denote the allocation returned by Algorithm~\ref{alg:learnplayfair}
after $\Tmax$ rounds. Then with probability at least $1-\delta$, 
$\util(\hat{\v}) > \util(\vopt^{\alpha-4D})-4D\R\V.$
\end{lemma}
\begin{proof}
Consider the following optimization problem, $\mathcal{P}(\alpha, \{\lambda_i\}, \{\bar{\lambda}_i\}, \V)$.
\begin{align*}
&\max_{\v} &&\util\left(\v, \{\C(\lambda_i)\}\right), \\
&\text{subject to} &&\left|\f_i\left(\vi, \C(\bar{\lambda}_i)\right) - \f_j\left(\vj,  \C(\bar{\lambda_i})\right)\right| \le \alpha, \forall i \text{ and }j,\\ 
& &&\sum_{i \in [\R]} \vi \le \V,\\
& &&\vi\in\N.
\end{align*} 
We can think of the above optimization problem as the case
where the underlying candidate distributions used for the objective value and the fairness
constraints are different. Let us write $\alloc(\alpha, \{\lambda_i\}, \{\bar{\lambda}_i\}, \V)$ to denote an 
optimal allocation in the above optimization problem, $\mathcal{P}(\alpha, \{\lambda_i\}, \{\bar{\lambda}_i\}, \V)$.
So an optimal fair allocation and the allocation returned by Algorithm~\ref{alg:learnplayfair}
can be written as $\alloc(\alpha, \{\lambda^*_i\},\{\lambda^*_i\}, \V)$ 
and $\alloc(\alpha, \{\hat{\lambda}_i\}, \{\hat{\lambda}_i\},\V)$, respectively. 

Note that for any fixed allocation $\v$
\begin{equation}
\label{eq:x}
\left|\util(\v, \{\C(\lambda^*_i)\}) - \util(\v, \{\C(\hat{\lambda}_i)\})\right| = 
\left|\sum_{i \in \R} \sum_{\c = 0}^{\vi} \T^*_i(c) - \sum_{i \in \R} \sum_{\c = 0}^{\vi} \hat{\T}_i(c)\right|\le 2\R\V D,
\end{equation}
where $\hat{\T}_i$ is the tail probability of $\C(\hat{\lambda}_i)$. 
This is because $|\T^*_i(c) - \hat{\T}_i(c)| \le 2D_{TV}(\C(\lambda_i), \C(\hat{\lambda}_i))$. 
In other words, even when the underlying candidate distribution changes 
for the objective value, an allocation value can change by at most $2\R\V D$. 

Now observe that
\begin{align*}
\util(\hat{\v}, \{\C(\lambda^*_i)\}) \geq \util(\hat{\v}, \{\C(\hat{\lambda}_i)\}) - 2\R\V D &= 
\util\left(\alloc\left(\V,\{\C(\hat{\lambda}_i)\}, \{\C(\hat{\lambda}_i)\}, \alpha\right), \{\C(\hat{\lambda}_i)\}\right)-2\R\V D \\
&\geq  \util\left(\alloc\left(\V,\{\C(\hat{\lambda}_i)\}, \{\C(\lambda^*_i)\}, \alpha-4D\right),\{\C(\hat{\lambda}_i)\}\right)-2\R\V D\\
&\geq  \util\left(\alloc\left(\V,\{\C(\lambda^*_i)\}, \{\C(\lambda^*_i)\}, \alpha-4D\right), \{\C(\lambda^*_i)\}\right)-4\R\V D\\
& =\util(\vopt^{\alpha-4D})-4D\R\V.
\end{align*}
The inequalities in the first and third lines are by Equation~\ref{eq:x}, which shows how the utility 
deteriorates when the underlying distribution for the objective function changes. The inequality 
in the second line follows from Lemma~\ref{lem:close-fairness-index}, as any $(\alpha-4D)$ fair 
allocation is a feasible allocation to $\mathcal{P}(\V,\{\C(\hat{\lambda}_i)\}, \{\C(\hat{\lambda}_i)\}, \alpha)$, 
and $\alloc(\alpha, \{\lambda_i\}, \{\bar{\lambda}_i\}, \V)$ is an optimal solution to this problem.
\end{proof}
\begin{rmk}
\label{rm:gen}
Although we assumed Poisson distributions in this section, all our results hold for any 
single-parameter Lipschitz-continuous distribution whose parameter is drawn from a 
compact set. However, the convergence rate of Theorem~\ref{thm:convgoptfair} depends on the 
quantity $\eta(\epsilon)$ which \emph{depends} on the family of distributions used to model
the candidate distributions.
\end{rmk} 

\section{Experiments}
\label{sec:exp}
In this section, we apply our allocation and learning algorithms for the precision model 
to the Philadelphia Crime Incidents dataset, and complement the theoretical convergence guarantee 
of Algorithm~\ref{alg:learnplayfair} to an optimal fair allocation with empirical 
evidence suggesting fast convergence in practice. We also study the empirical trade-off between
fairness and utility in the dataset.
\subsection{Experimental Design}
\label{sec:exp-design}
The Philadelphia Crime Incidents
dataset\footnote{\url{https://www.opendataphilly.org/dataset/crime-incidents} 
accessed 2018-05-16.} contains
all the crimes reported to the Police Department's INCT system between 
2006 and 2016.
The crimes are divided into two types. Type I crimes include violent offenses 
such as aggravated assault, rape, and arson among others. 
Type II crimes include simple assault, prostitution, gambling and fraud.
For simplicity, we aggregate all crime of both types, but in practice, an actual police 
department would of course treat different categories of crime differently.
We note as a caveat that these incidents are \emph{reported} and may not
represent the entirety of committed crimes.

\begin{figure}[ht!]
\centering
\includegraphics[width=9.5cm]{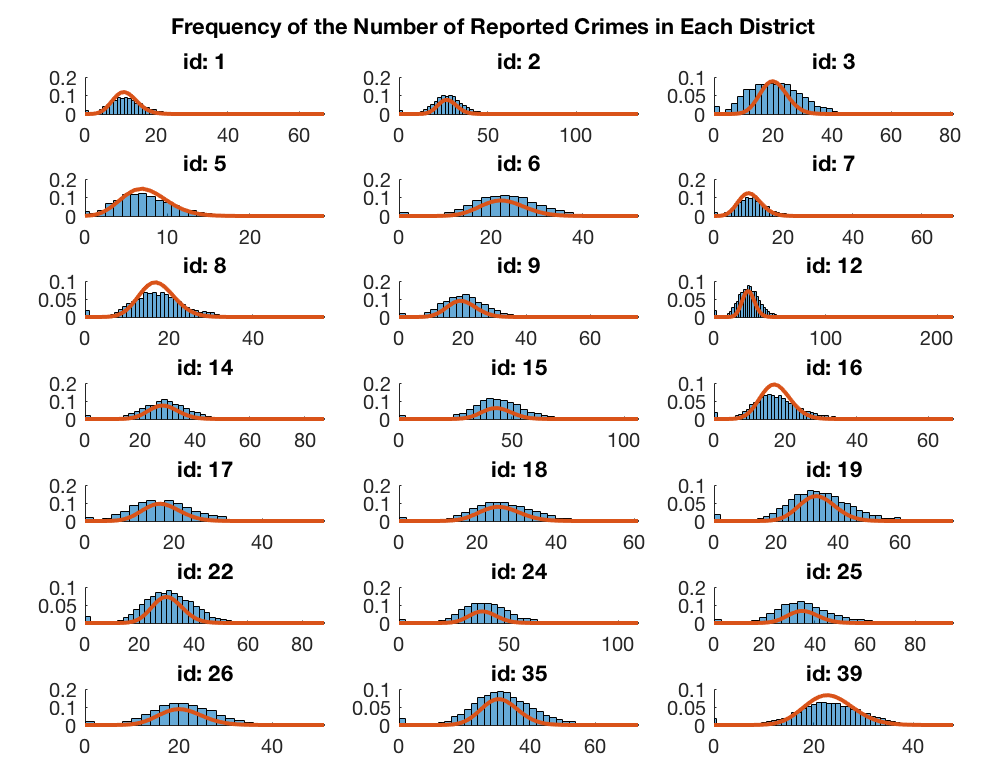}
\caption{Frequencies of the number of reported crimes in each district
in the Philadelphia Crime Incidents dataset. The red curves display the best Poisson
fit to the data.
\label{fig:dists-philly-fitting}}
\end{figure}

To create daily crime frequencies in Figure~\ref{fig:dists-philly-fitting},
we first calculate the daily counts of criminal incidents in each of the 21 geographical 
police districts in Philadelphia
by grouping together all the crime reports with the same date; we then normalize these counts 
to get frequencies.\footnote{The current list of 21 districts can
be found at \url{https://www.phillypolice.com/districts-units/index.html}.
The dataset however contains 25 districts from which we removed 4 from consideration.
Districts with identifiers 77 and 92 correspond to the airport and parks, so the
crime incident counts in these districts are significantly different from
 the rest of the districts. Moreover, we
removed districts with identifiers 4 and 23 which were both dissolved in 2010.}
Each subfigure in Figure~\ref{fig:dists-philly-fitting}
represents a police district. The horizontal axis of the subfigure corresponds to the number
of reported incidents in a day and the vertical axis represents the frequency of each 
number on the horizontal axis.
These frequencies approximate the true (marginal) distributions of the number of
reported crimes in each of the districts in Philadelphia. Therefore, throughout this section we take
these frequencies as the \emph{ground truth} candidate distributions for the number of reported
incidents in each of the districts.

Figure~\ref{fig:dists-philly-fitting} shows that crime distributions in different districts
can be quite different; e.g., the average number of daily reported
incidents in District 15 is 43.5, which is much higher than the average of 11.35 in District 1
(see Table~\ref{table:avg} in Appendix~\ref{sec:omitted-details-exp} for more details).
Despite these differences, each of the
crime distributions can be approximated well by a Poisson distribution.
The red curves overlayed in each subfigure correspond to the
Poisson distribution obtained via maximum likelihood estimation on data from that district. 
Throughout, we refer to such distributions as the \emph{best Poisson fit} to the data
(see Table~\ref{table:fit} in Appendix \ref{sec:omitted-details-exp} for details
about the goodness of fit).

In our experiments, we take the police officers assigned to the
districts as the resource to be distributed, the ground truth crime frequencies as candidate 
distributions, and aim to maximize the sum of the number of crimes discovered under 
the precision model of discovery.
\subsection{Results}
\label{sec:exp-results}

We can quantify the extent to which fairness degrades utility in the dataset through a notion we call
\emph{Price of Fairness} (PoF henceforth). In particular, given the ground truth crime distributions
and the precision model of discovery,
for a fairness level $\alpha$, we define $\text{PoF}(\alpha)=\util(\vopt^*)/\util(\vopt^{\alpha})$.
The PoF is simply the ratio of the expected number of crimes discovered by an optimal allocation to
the expected number of crimes discovered by an optimal $\alpha$-fair allocation.
Since $\util(\vopt^*)\geq \util(\vopt^{\alpha})$ for all $\alpha$,
the PoF is at least one. Furthermore, the PoF is monotonically non-increasing in $\alpha$.
We can apply the algorithms given in
Sections~\ref{sec:min-model-opt}~and~\ref{sec:min-model-opt-fair} respectively for computing optimal 
unconstrained, and optimal fair allocations with the
with ground truth distributions as input and numerically compute the PoF.
This is illustrated in Figure~\ref{fig:pof}. The $x$ axis corresponds to different $\alpha$ values
and the $y$ axis displays $1/\text{PoF}(\alpha)$.
Each curve corresponds to a different number of total police
officers denoted by $\V$. Because feasible allocations must be integral, there can sometimes be no 
feasible $\alpha$-fair allocation for small $\alpha$. 
Since the PoF in these cases is infinite we instead opt to display the inverse, $1/\text{PoF}$, which 
is always bounded in $[0,1]$.
Higher values of inverse PoF are more desirable.

\begin{figure}[ht!]
\centering
\includegraphics[width=9.5cm]{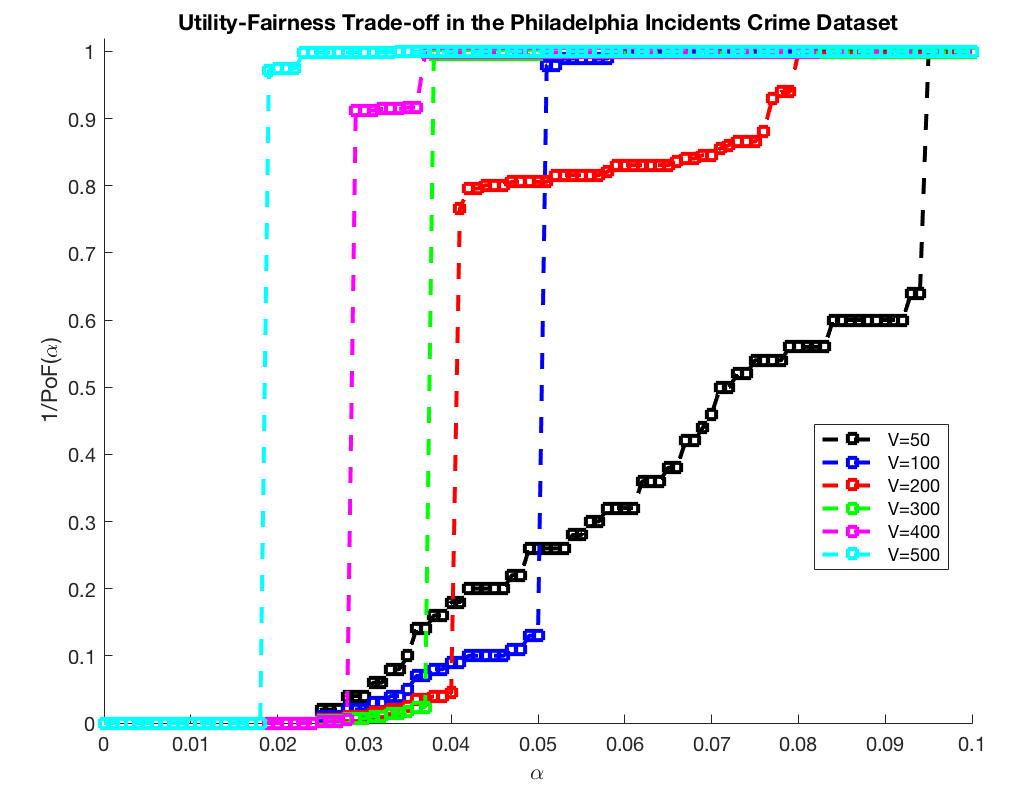}
\caption{\label{fig:pof}
Inverse PoF plots for the Philadelphia Crime Incidents dataset.
Smaller values indicate greater sacrifice in utility to meet the fairness constraint.}
\end{figure}

Figure~\ref{fig:pof} shows a diverse set of utility/fairness trade-offs depending on the number
of available police officers. It also illustrates that the cost of fairness is rather low in most regimes.
For example, in the worst case, with only 50 police officers (the black curve) (which is much smaller 
than the average number of daily reported crimes: 563.88), the inverse PoF is 1 for $\alpha \geq 0.1$,
which corresponds to a 10\% difference
in the discovery probability across districts.
When we increase the number of available police officers to 400 (the magenta curve), tolerating only a 4\% difference
in the discovery probability across districts is sufficient to guarantee no loss in the utility. 
Figure~\ref{fig:pof} also shows that for any fixed $\alpha$, the inverse $\text{PoF}(\alpha)$
tends to increase as the number of police increases (i.e. the cost of fairness decreases).\footnote{There
are exceptions to this observation -- for example, in the regime when $\alpha$ is between 0.03 and 0.04, the inverse PoF
decreases as $\V$ increases from 100 to 200. This occurs because only integral allocations are feasible, 
so achieving a particular fairness level may require leaving some resources unallocated until 
significantly more resources become available;  increasing $\V$ in this regime improves the 
utility of an optimal allocation while leaving the utility of an optimal fair allocation unchanged.}
This captures the intuition 
that fairness becomes a less costly constraint when resources are in greater supply. 
Finally, we observe a thresholding phenomenon in Figure~\ref{fig:pof};
in each curve, increasing $\alpha$ beyond a threshold will significantly increase the inverse PoF.
This is due to discretization effects, since only integral allocations are feasible.

We next turn into analyzing the performance of Algorithm~\ref{alg:learnplayfair}
in practice. We run the algorithm instantiated to fit Poisson distributions, but use observations from the ground 
truth distribution at each round. As we have shown in Figure~\ref{fig:dists-philly-fitting} and Table~\ref{table:fit}, 
the ground truth is well approximated by a Poisson distribution.

We measure the performance of Algorithm~\ref{alg:learnplayfair} as follows.
First, we fix a police budget $\V$ and unfairness budget $\alpha$ and run Algorithm~\ref{alg:learnplayfair}
for 2000 rounds using the dataset as the ground truth. That is, we simulate each round's crime count realizations
in each of the districts as being sampled from the ground truth distributions, and return censored observations 
under the precision model
to Algorithm~\ref{alg:learnplayfair} according to the algorithm's allocations and the drawn realizations.
The algorithm returns an allocation after termination  and we can measure the expected number of crimes discovered
and fairness violation (the maximum difference in discovery probabilities
over all pairs of districts) of the returned allocation
using the ground truth distributions. Varying $\alpha$ while fixing $\V$ allows us to trace out the Pareto frontier
of the utility/fairness trade-off for a fixed
police budget. Similarly, for any fixed $\V$ and $\alpha$, we can run Algorithm~\ref{alg:opt-fair-dark-pool}
(the offline algorithm for computing an optimal fair allocation) with the ground truth distributions as input and trace 
out a Pareto curve by varying $\alpha$.
We refer to these two Pareto curves by the \emph{learned} and \emph{optimal} Pareto curves, 
respectively.\footnote{We can also generate
\emph{fitted} Pareto curves using best Poisson fit distributions instead of the ground truth distributions. These curves look
very similar to the optimal Pareto curves (see Figure~\ref{fig:Paretotrue}
in Appendix~\ref{sec:omitted-details-exp}).}
So to measure the performance of Algorithm~\ref{alg:learnplayfair}, we can compare the learned
and optimal Pareto curves.
\begin{figure}[ht!]
\centering
\includegraphics[width=9.5cm]{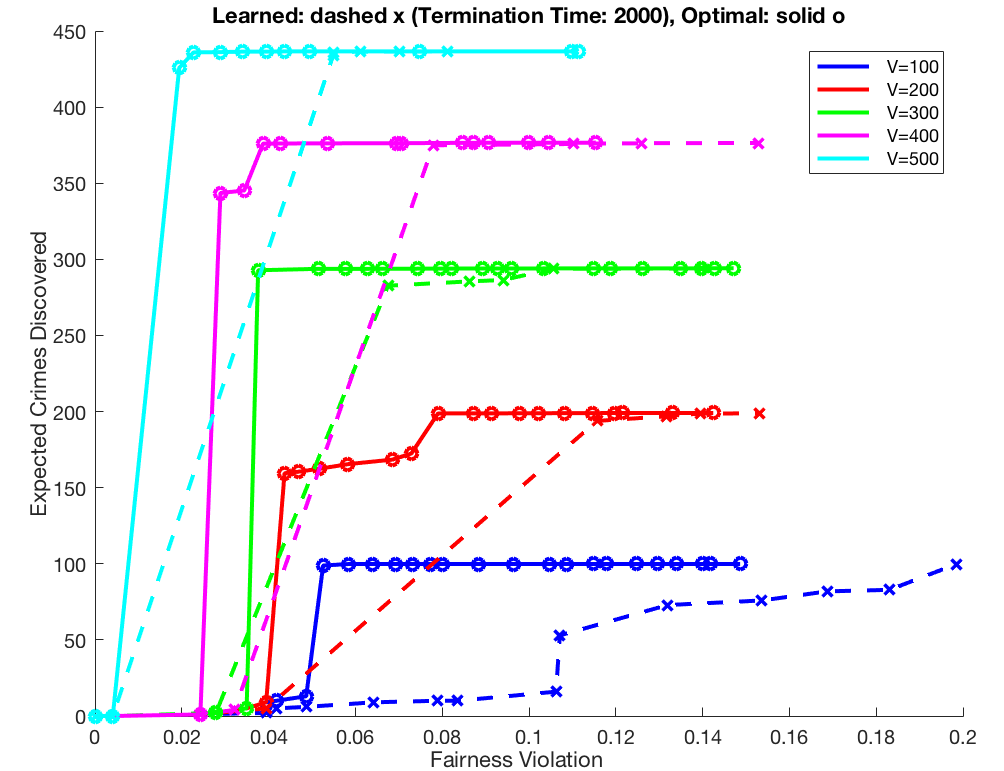}
\caption{
Pareto frontier of expected crimes discovered  versus fairness violation.
\label{fig:Paretolearning}}
\end{figure}

In Figure~\ref{fig:Paretolearning}, each curve corresponds to a police budget. The $x$ and $y$ axes
represent the expected number of crimes discovered and fairness violation for allocations
on the Pareto frontier, respectively. In our simulations we varied $\alpha$ between 0 and 0.15.
For each police budget $\V$, the `x' s connected by the dashed lines
show the learning Pareto frontier.  Similarly, the circles connected by solid lines show the optimal Pareto frontier.
We point out that while it is possible for the fairness violations in the learned Pareto curves to be higher than the level
of $\alpha$ set as an input to Algorithm~\ref{alg:learnplayfair}, the fairness violations in the optimal Pareto curves
are always bounded by $\alpha$.

The disparity between the optimal and learned Pareto curves are due to the fact that the 
learning algorithm has not yet fully converged. This can be attributed to the large number of
censored observations received by Algorithm~\ref{alg:learnplayfair}, which are significantly 
less informative than uncensored observations.
Censoring happens frequently because the number of police used in every case plotted is 
less than the daily average of 563.88
crimes across all the districts in the dataset --- so it is unavoidable that in any allocation, there 
will be significant censoring in at least some districts. 

Figure~\ref{fig:Paretolearning} shows that while the learning
curves are dominated by the optimal curves, the performance of the learning algorithm approaches
 the performance of the offline optimal allocation as $\V$ increases. Again, this is because increasing $\V$ generally 
 has the effect of decreasing the frequency of censoring.

We study the $\V = 500$ regime in more detail, to explore the empirical rate of convergence. 
In Figure~\ref{fig:fastconvergence}, we study the round by round performance of the allocation 
computed by Algorithm~\ref{alg:learnplayfair}
in a single run with the choice of $\V=500$ and $\alpha = 0.05$.
\begin{figure}[ht!]
\centering
\includegraphics[width=9.5cm]{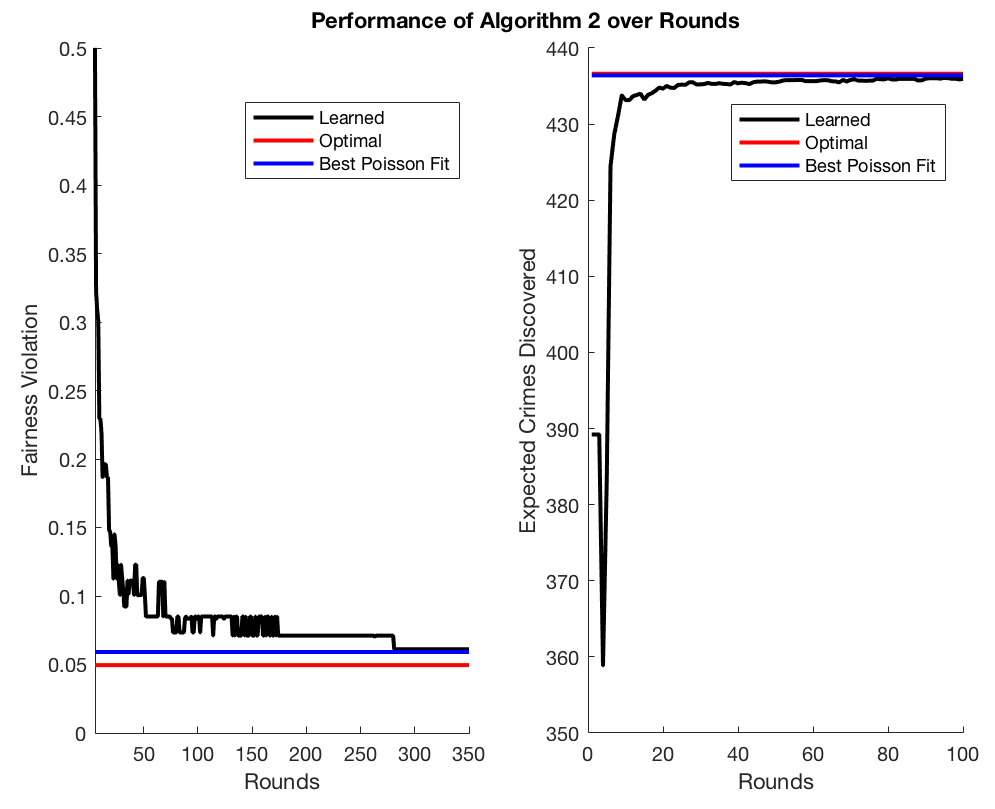}
\caption{The per round expected number of crimes discovered and fairness violation of Algorithm~\ref{alg:learnplayfair}.
$\V=500$ and $\alpha = 0.05$.
\label{fig:fastconvergence}}
\end{figure}

In Figure~\ref{fig:fastconvergence}, the $x$ axis labels progression of rounds of the algorithm. The $y$ axis
measures the fairness violation (left) and expected number of crimes discovered (right)
of the allocation deployed by the algorithm, as measured with respect to the ground truth distributions.
The black curves represent Algorithm~\ref{alg:learnplayfair}.
For comparison we also show the same quantities for an offline optimal fair allocation as 
computed with respect to the ground truth (red line), and an offline
optimal fair allocation as computed with respect to the best Poisson fit to the ground truth (blue line).
Note that in the limit, the allocations chosen by Algorithm~\ref{alg:learnplayfair} are guaranteed to converge to the
blue baselines --- but not the red baseline, because the algorithm is itself learning a Poisson approximation to 
the ground truth. The disparity between the red and blue lines quantifies the degradation
in performance due to using Poisson approximations, rather than due to non-convergence of the learning process.

Figure~\ref{fig:fastconvergence} shows that Algorithm \ref{alg:learnplayfair} converges to the 
Poisson approximation baseline well before the termination time of 2000, and substantially before 
the convergence bound guaranteed by our theory. 
Examining the estimated Poisson parameters used internally by Algorithm~\ref{alg:learnplayfair} 
reveals that although the allocation has converged to an optimal fair allocation, the estimated parameters 
have not yet converged to the parameters of the best Poisson fit in any of the districts.
In particular, Algorithm~\ref{alg:learnplayfair} underestimates the parameters in all of the districts --- 
but the degree of the underestimation is systematic: the correlation coefficient between the true and
estimated parameters is $0.9975$.

We see also in Figure~\ref{fig:fastconvergence} that convergence to the optimum expected number of discovered crimes
occurs more quickly than convergence to the target fairness violation level.
This is also apparent in Figure~\ref{fig:Paretolearning}
where the learning and optimal Pareto curves are generally similar in terms of the maximum number of crimes discovered,
while the fairness violations are higher in the learning curves.
\section{The Random Discovery Model}
\label{sec:urn}
Finally, we consider the \textit{random model} of
discovery. In the random model, when $\vi$ units are allocated to a group with $\ci$ candidates, 
the number of discovered candidates is a random variable corresponding to the number of candidates 
that appear in a uniformly random sample of $\vi$ individuals from a group of size $\m_i$. Equivalently, 
when $\vi$ units are allocated to a group of size $\m_i$ with $\ci$ candidates,
the number of candidates discovered by $\disc(\cdot)$ is a random variable
$\disc(\vi, \ci) \triangleq \caughti,$
where $\caughti$ is drawn from the hypergeometric distribution with parameters
$\m_i$, $\ci$ and $\vi$.
Furthermore, the expected number of candidates discovered when allocating
$\vi$ units to group $i$ is
$\E[\disc(\vi, \ci)] = \vi\E[\ci]/\m_i$.

For simplicity, throughout this section, we assume $\m_i \geq \V$ for all $i$.
This assumption can be completely relaxed (see~the discussion in 
Appendix~\ref{sec:omitted-details-urn}).
Moreover, let $\mu_i = \E[\ci]/\m_i$ denote the expected fraction of 
candidates in group $i$. Without loss of generality, for the rest of this section,
we assume $\mu_1\geq \mu_2 \geq \ldots \geq \mu_\R$. 
\subsection{Optimal Allocation}
\label{sec:urn-opt}
In this section, we characterize optimal allocations.
Note that the expected number of candidates discovered by the allocation choice $\vi\leq \m_i$ in
group $i$ is simply $\vi\mu_i$.  This suggests a simple algorithm to
compute  $\vopt^*$: allocating every unit of the resource
to group 1.
More generally, let $\R^* = \{i \mid \mu_i = \mu_1\}$ denote the
subset of groups with the highest expected number of candidates. An allocation is optimal 
if and only if it only allocates \emph{all} resources to groups in $\R^*$. 
\subsection{Properties of Fair Allocations}
\label{sec:urn-opt-fair}

We next discuss the properties of fair allocations in the random discovery model.
First, we point out that the discovery probability can be simplified as
 $$\f_i({\vi})= \E_{\ci\sim\C_i}\left[\frac{\ci \vi /\m_i}{\ci}\right] = \frac{\vi}{\m_i}.$$
So an allocation is $\alpha$-fair in the random model
if $|\vi/\m_i-\vj/\m_j|\leq \alpha$ for all groups $i$ and $j$.
Therefore, fair allocations (roughly) distribute resources
in proportion to the size of the groups, essentially ignoring the candidate distributions 
within each group. 
We defer the full characterization 
to Appendix~\ref{sec:omitted-details-urn}.

\subsection{Price of Fairness}
\label{sec:urn-pof}
Recall that PoF quantifies the extent to which constraining the allocation to
satisfy $\alpha$-fairness degrades  utility. While in Section~\ref{sec:exp}
we study the PoF on the Philadelphia Crime Incidents dataset, we can
define a worst-case variant as follows.

\begin{definition}
\label{def:pof}
Fix the random model of crime discovery and let $\alpha\in[0,1]$.
We define the PoF as
$$\text{PoF}(\alpha) = \max_{\C} \frac{\util(\vopt^*, \C)}{\util(\vopt^{\alpha}, \C)}.$$
where $\C$ ranges over all possible candidate distributions.
\end{definition}

We can fully characterize this worst-case PoF in the random discovery model. We defer the proof
of Theorem~\ref{thm:urn-pof} to Appendix~\ref{sec:omitted-details-urn}.
\begin{thm}
\label{thm:urn-pof}
The PoF in the random discovery model is $$\text{PoF}(\alpha) = \begin{cases} 1,  & \frac{\V}{\m_1}\leq \alpha,  \\
\frac{M}{m_1+\alpha(M-m_1)},  & \frac{\V}{\m_1}> \alpha.\end{cases} $$
\end{thm}
The PoF in the random model can be as high as $M/\m_1$ in the worst case. If all groups are 
identically sized, this grows linearly with the number of groups. 
\section{Conclusion and Future Directions}
\label{sec:conclusion}

Our presentation of allocative fairness provides a family of fairness definitions, 
modularly parameterized by a ``discovery model''. What counts as ``fair'' depends a 
great deal on the choice of discovery model, which makes explicit what would otherwise be 
unstated assumptions about the process of tasks like policing. The random and precision 
models of discovery
studied in this paper represent two extreme points of a spectrum. In the predictive policing 
setting, the random model of discovery assumes that officers have no advantage over random 
guessing when stopping individuals for further inspection. The precision model assumes they 
can oracularly determine offenders, and stop only them. An interesting direction for future work 
is to study discovery models that lie in between these two.

We have also made a number of simplifying assumptions that could be relaxed. 
For example, we assumed the candidate distributions
are \emph{stationary} --- fixed independently of the actions of the algorithm. 
Of course, the deployment of police officers can \emph{change} crime 
distributions. Modeling this kind of dynamics, and designing learning 
algorithms that perform well in such dynamic settings would be interesting. 
Finally, we have assumed that the same discovery model applies to all groups. 
One friction to fairness that one might reasonably conjecture is that the discovery 
model may differ between groups --- being closer to the precision model for one group, 
and closer to the random model for another. We leave the study of these extensions to future work.

\subsection*{Acknowledgements}
We thank Sorelle Friedler for giving a talk at Penn which initially 
inspired this work. We also thank Carlos Scheidegger, Kristian Lum, Sorelle Friedler, 
and Suresh Venkatasubramanian for helpful discussions at an early stage of this work. 
We thank Richard Berk and Greg Ridgeway for helpful discussions about 
predictive policing. Finally, we thank the anonymous reviewer for helpful comments
regarding Figure~\ref{fig:pof}.
\bibliographystyle{plainnat}
\bibliography{bib}
\appendix
\section{Feasibility in Expectation}
\label{sec:feasible-in-exp}
In this section, we show how to compute $\vopt^{\alpha}$ 
for any \emph{arbitrary} but \emph{known} candidate distributions 
$\C$ and known discovery model $\disc(\cdot)$ in a relaxation where the 
feasibility constraint is satisfied in expectation.

The first observation is that when $\disc(\cdot)$ and $\C$ are both known,
for a group $i$ and allocation of $j$ units of resource 
to that group, the expected number of discovered candidates
$$\disc_{ij} = \E_{\ci\sim\C_i}\left[\disc(j,\ci)\right],$$ and 
the discovery probability 
$$\f_{ij} = \E_{\ci\sim\C_i}\left[\frac{\disc(j,\ci)}{\ci}\right],$$ 
can both be computed exactly. The second observation is that when allowing the 
feasibility
condition to be satisfied in expectation, instead of allocating
integral units of resources to each group, we can allocate
resources to a group using a distribution. 

Let $p_{ij}$ denote the probability that $j$ units of resource is allocated to group $i$.
We can compute $\vopt^{\alpha}$ by writing the following linear program with $p_{ij}$s as variables.
\begin{align*}
& \max_{p_{ij}} && \sum_{i\in[\R]}\sum_{j=1}^{\V}p_{ij} \disc_{ij},\\
& \text{subject to} &&\sum_{i\in[\R]}\sum_{j=1}^{\V} p_{ij} j \leq \V,\\
& &&\left|\sum_{j=1}^{\V} p_{ij}\f_{ij}-\sum_{j=1}^{\V} p_{i'j}\f_{i'j}\right|  
\leq \alpha, \forall i \text{ and } i'\in [\R],\\
& &&\sum_{j=1}^{\V} p_{ij}  = 1, \forall j,\\
& && p_{ij}\geq 0, \forall i \text{ and } j.
\end{align*}

The objective function maximizes the number of candidates discovered given 
the allocation. The first 
constraint guarantees that the allocation is feasible in expectation.
The second constraint (which is linear in $p_{ij}$) ensures that $\alpha$-fairness is satisfied 
by the allocation. The last two constraints guarantees that for any $i$, 
$p_{ij}$ values define a valid probability distribution on all the possible allocations
to group $i$.
\section{Omitted Details from Section~\ref{sec:min-model}}
\label{sec:dark-pool}
\subsection{Omitted Details from Section~\ref{sec:min-model-opt}}
We first show how the expected number of discovered candidates in a group in the precision
model can be written as a function of the tail probabilities of the group's candidate
distribution.
\begin{lemma}[\citet{GanchevKNV09}]
\label{lem:dark-pool-obj}
The expected number of discovered candidates in the precision model when allocating $\vi$
units of resource to group $i$ can be written as 
$\E_{\ci\sim\C_i}[\min\left(\ci, \vi\right)] = \Sigma_{c=1}^{\vi}\T_i(c)$.
\end{lemma}
\begin{proof}
	\begin{align*}
	\E_{\ci\sim\C_i}\left[\min(\ci,\vi)\right] 
	&= \sum\limits_{c=1}^{\m_i} \Pr_{\ci\sim\C_i}\left[\ci=c\right]\min(c,\vi)= \sum\limits_{c=1}^{\vi-1} \Pr_{\ci\sim\C_i}\left[\ci=c\right] c+ \vi\T_i(\vi)\\
	&= \sum\limits_{c=1}^{\vi-2} \Pr_{\ci\sim\C_i}\left[\ci=c\right] c+ (\vi-1)\T_i(\vi-1)+\T_i(\vi)
	= \T_i(1) + \dots + \T_i(\vi-1) + \T_i(\vi)\\
	&= \sum\limits_{c=1}^{\vi}\T_i(c).
	\end{align*}
Note that we can perform the telescoping in the 3rd and 4th lines by observing that 
$\Pr_{\ci\sim\C_i}\left[\ci = c-1\right]+\T_i(c) = \T_i(c-1).$	
\end{proof}

We then show that a greedy algorithm would find an optimal allocation
in the precision model when the candidate distributions are known.
\begin{thm}[Theorem 1 in \citet{GanchevKNV09}]
	\label{thm:pf-dkpl-opt}

	The allocation returned by greedily allocating the next unit of resource to a 
	group in \[  \argmax\limits_{i\in[\R]} \left(\T_i(\vit+1)-\T_i(\vit)\right),  \] where $\vit$ 
	is the current allocation to group $i$ at the $t$\textsuperscript{th} round 
	maximizes the expected number of candidates discovered.
\end{thm}	 
\begin{proof}
	Since the tail probability functions $\T_i(c)$ are all non-increasing 
	(that is, for $c\leq c'$, we have $\T_i(c')\leq \T_i(c)$), the greedy allocation 
	returns an allocation $\v$ which maximizes 
	\[\util(\v)=\sum\limits_{i\in[\R]}\sum\limits_{c=1}^{\vi}  \T_i(c) 
	\text{ such that } \sum_{i\in[\R]}\vi = \V.\]

Using Lemma~\ref{lem:dark-pool-obj}
we have that 
$$
\sum\limits_{i\in[\R]}\E_{\ci\sim\C_i}\left[\min\left(\ci, \vi\right)\right] = \sum\limits_{i\in[\R]}\sum_{c=1}^{\vi}\T_i(c).
$$
So the above double-summation is exactly equal 
to the expected number of discovered candidates. 
To see that the greedy solution is optimal, notice that any solution which does not allocate the 
marginal resource to the tail with the highest remaining probability can be improved by 
reallocating the final allocated resource in some lower tail probability group to the one in 
the higher tail probability.
Finally since each term in the objective
function is non-negative an optimal allocation would use all the $\V$ units of resource (so
the feasibility constraint is tight).
\end{proof}

\subsection{Omitted Details from Section~\ref{sec:min-model-opt-fair}}
\ifproofs
\else
\begin{proof}[Proof of Theorem~\ref{thm:fair-opt-min}]
Fix an optimal $\alpha$-fair allocation $\vopt^{\alpha}$. 
In $\vopt^{\alpha}$, some group $i$ has the highest $\f_i$ and
receives allocation $w^\alpha_i$.
Suppose we know $i$ and $w^{\alpha}_i$ (we relax this assumption
at the end of the proof). Using the knowledge of $\C_i$
we can compute $f_i(w^\alpha_i)$. This implies that $\f_j \in [\f_i-\alpha, \f_i]$ for every other
group $j$, which in turn can be
used to derive the set of all possible allowable allocations $w^\alpha_j$
which do not violate $\alpha$-fairness.
	
We claim that if we initialize the allocation $w^\alpha_j$ to be the lower
bound of the interval corresponding to group $j$,
then greedily assign the surplus units with the added restriction
that $w^\alpha_j$ is always inside of its respective interval,
we achieve an optimal $\alpha$-fair allocation.
	
Since we assume we know $w^\alpha_i$ to be the allocation to group $i$ in an optimal
fair allocation, this allocation must be achievable by picking some value from each 
of the intervals,
thus initializing the allocation to the lower bound of each interval certainly cannot assign
more than $\V$ units in total.  By the same argument as for the unconstrained
greedy algorithm,
since the objective function is concave (recall that the tail probabilities
are non-increasing) and increasing in each argument $w^\alpha_j$, a greedy search
over this feasible region finds the desired allocation.
	
The algorithm does not know \textit{a priori} the group $i$
which has the maximum $f_i$ or $w^\alpha_i$, so
it must search over these options.
There are $\R$ guesses for group $i$ and $\V+1$ guesses for the
allocation to the group.
So there are a total of $\R(\V+1)$ guesses
that needs to be considered. For each guess, it takes
$O(M)$ to compute the upper and lower bounds on the allocation to each of
the groups and $O(\R\V)$ to run the greedy algorithm. 
So the running time of Algorithm~\ref{alg:opt-fair-dark-pool} is $O(\R\V(\R\V+M))$.
\end{proof} 
\fi

\subsection{Omitted Details from Section~\ref{sec:min-model-learning}}
\ifproofs
\else
\begin{proof}[Proof of Theorem~\ref{thm:imp}]

\end{proof}
\fi
\subsection{Omitted Details from Section~\ref{sec:mle-poisson}}
Since in our analysis we consider the groups separately, 
we fix a group $i$ throughout the rest of this
section and drop the subscript $i$ for convenience.
Our first lemma shows that the true underlying parameter $\lambda^*$ uniquely 
maximizes $\E[\ell]$ for any allocation. 
Since $\like^*$ is just a sum of $\E[\ell]$ terms, it follows as a corollary that $\like^*$ is 
uniquely maximized at $\lambda^*$ for any sequence of allocations.
This is stated as Lemma~\ref{lem:uniquemax} and Corollary~\ref{cor:1}.
\begin{lemma}
\label{lem:uniquemax}
For any $\vsingle$, 
$\argmax_{\lambda} \E_{\caught}[\ell(v,\caught; \lambda)] = \{\lambda^*  \}.$
\end{lemma}
\begin{proof}
Notice that since the expected log-likelihood function is the average over 
time periods of individual $\ell(\vit, \cit, \lambda)$ terms, $\lambda^*$ being 
the unique maximizer of each term individually will imply that it is the unique 
maximizer of the the expected log-likelihood function. Thus we aim to show that
\begin{align*}
\E\left[\ell\left(\vsingle^t,\caughtt, \lambda^*\right) \right] &>  
\E\left[\ell\left(\vsingle^t,\caughtt, \lambda\right)\right].
\end{align*}
Notice that this is true if and only if
\begin{align} \label{ineq:uniquemaxgoal}
\E \left[-\log\left(\frac{p(\vsingle^t,\caughtt,\lambda)}{p(\vsingle^t,\caughtt,\lambda^*)}\right)\right] >0.
\end{align}
Recall the Gibb's inequality, written here for the discrete case as in \citet{mackay}.
\begin{lemma}
\label{lem:gibbs}
Suppose $p$ and $q$ are two discrete distributions. Let $D_{KL}(p||q)$ denote the KL 
divergence between
$p$ and $q$. Then
$D_{KL}(p||q) \geq 0$
with equality if and only if $p(x)=q(x)$ for all $x$. 
\end{lemma}
	
So the quantity in Equation \ref{ineq:uniquemaxgoal} is the KL 
divergence between two distributions. 
Since the distributions place different probabilities on at least one event 
(in fact, infinitely many events), the inequality
is strict by Lemma~\ref{lem:gibbs}.
\end{proof}

\begin{cor} 
\label{cor:1}
For any $\hvec^t$,
$\argmax_{\lambda} \like^*(\lambda, \hvec^t) = \{\lambda^*\}.$
\end{cor}

\begin{lemma}
\label{lem:bound}
$\left|\ell\left(\vsingle^t, \caught^t;\lambda\right)\right| \le
\max\left(\left|\ell\left(\V, \V;\lambda_{\min}\right)\right| , 
\left|\ell\left(\V-1, \V; \lambda_{\min}\right)\right|, \left|\ell\left(1, 0; \lambda_{\max}\right)\right|\right).
$
\end{lemma}
\begin{proof}
The Poisson's PMF is unimodal and achieves its maximum at $\lambda$, where 
$p(\vsingle^t,\caught^t; \lambda)$ will be at most 1, meaning $\ell(\vsingle^t,\caught^t; \lambda) \le 0$. 
So, in order to bound the absolute value, we will bound how small $\ell$ can get. 
	
We prove the claim by a case analysis.
For uncensored observations, the minimum log-likelihood is achieved at either $0$ or at $\V-1$ 
due to unimodality.  In this case, the choice of $\lambda$ that can result in the minimum value is 
at $\lambda_{\max}$ or $\lambda_{\min}$, respectively. In the case of a censored observation, 
$\ell(\vsingle^t, \caught^t; \lambda) = \log(1- F(\vsingle^t-1; \lambda)).$ 
So the minimum will be achieved at $\ell(\V, \V; \lambda_{\min})$.  
\end{proof}

Next we show that for any fixed $\lambda$, with high probability over the randomness of 
$\{\caughtt\}$, $\hat{\like}$ converges to $\like^*$ for any sequence of allocations $\{v^t\}$
that Algorithm \ref{alg:learnplayfair} could have chosen. 

\begin{lemma}\label{lem:probabilisticsinglelambda}
For any $\lambda \in [\lambda_{min}, \lambda_{max}]$ and any $\histt$
$$\Pr\left[\left|\hat{\like}(\histt, \lambda) - \like^*(\histt, \lambda) \right| > \epsilon\right] 
\le 2e^{-\frac{t\epsilon^2}{2C^2}},$$
where $C$ is a constant and in the case of Poisson distribution
\begin{align*}C = \frac{1}{2} max\Big(\left|\ell\left(\V, \V; \lambda_{min}\right)\right| , 
\left|\ell\left(\V-1, \V; \lambda_{min}\right)\right|, \left|\ell\left(1,0; \lambda_{max}\right)\right|\Big).\end{align*}
\end{lemma}

\begin{proof}
With a slight abuse of notation, let $\alloc(\hist^s)$ denote the allocation to the group
we are considering.
We define $Q^t$ as follows.
\begin{align*}
Q^t &:= t \left(\hat{\like}\left(\histt, \lambda\right) - \like^*\left(\histt, \lambda\right)\right) 
 = \sum_{s=1}^{t} \ell\left(\alloc(\hist^s),\caught^s; \lambda\right) - 
 \sum_{s=1}^{t} \E\left[\ell\left(\alloc(\hist^s), \caught^s; \lambda\right)\right].
\end{align*}

So $Q^t$ is the sum of the difference between each period's observed and 
expected conditional log-likelihood function.  Notice that $Q^t$ is a 
martingale, as $\E[Q^{t+1} | Q^{t}] = Q^t$.  Moreover, its terms form 
a bounded difference sequence since $\ell(\alloc(h_s),\caught^s; \lambda)$ 
is continuous in $\caughtis$ with $\caughtis\in[0, \V]$ and $\lambda\in[\lambda_{\min},\lambda_{\max}]$. 
In particular, we show in Lemma~\ref{lem:bound}
that $\ell(\vsingle^t,\caughtt; \lambda)| \le 2C$. 

Since $\{Q^t\}$ is a bounded martingale difference sequence, 
we can apply Azuma's inequality to get
\begin{align*}
\Pr\left[\left|Q^{t} - Q^0\right| \ge t\epsilon\right] \le  2e^{-\frac{t\epsilon^2}{2C^2}}.
\end{align*}
Rearranging gives the claim.

\end{proof}

For $k$ values of $\lambda$, taking the union bound and setting $\epsilon = 
\sqrt{2C^2\ln(2k\R/\delta)/t}$ provides the following corollary.
\begin{cor} \label{cor:unionbound}
Let $\Lambda$ be a set of $k$ values such that for any $\lambda \in \Lambda$,
$\lambda \in [\lambda_{\min},\lambda_{\max}]$. Then
with probability at least $1-\delta/\R$
$$\max_{\lambda \in \Lambda} \left|\hat{\like}(\histt, \lambda) - \like^*(\histt, \lambda) \right| \le \sqrt{\frac{2C^2\ln(\frac{2k\R}{\delta})}{t}},$$
where $C$ is as in Lemma \ref{lem:probabilisticsinglelambda}.
\end{cor}
\ifproofs
We now need to show that the likelihood functions are Lipschitz-continuous:
\begin{lemma} \label{lem:lipschitz}
For any $\lambda$, $\lambda' \in [\lambda_{\min},\lambda_{\max}]$ such 
that $|\lambda-\lambda'|<\epsilon$, we have that $|\ell(\vit, \caughtit, \lambda)-\ell(\vit, \caughtit, \lambda')| 
\le b \epsilon$ for some constant $b$.
\end{lemma}
\begin{proof}
Recall that a differentiable function is Lipschitz-continuous if and only if its derivative is bounded.
By definition,
	
\begin{align*}
\ell\left(\vsingle^t\caught^t; \lambda\right) := \begin{cases} 
\log \left(\frac{e^{-\lambda} \lambda^{\caught^t}}{\caught^t!}\right), & \caught^t< \vsingle^t, \\
\log\left(1-F\left(\vsingle^t-1; \lambda\right)\right), & \text{otherwise.}
\end{cases}
\end{align*}
So we can analyze the derivative by cases. 
	
In the uncensored case ($\caught^t < \vsingle^t$), we have that
\begin{align*}
\ell\left(\vsingle^t,\caught^t, \lambda\right) = 
-\lambda + \caught^t \log \lambda - \log \caught^t! \implies 
\frac{\partial  \ell}{\partial \lambda}  = -1 + \frac{\caught^t}{\lambda}.
\end{align*}
For $\lambda \in [\lambda_{\min},\lambda_{\max}]$ with $\lambda_{\min}>0$, 
this function is continuous and its domain is bounded. Hence its image is bounded and

In the censored case ($\caught^t= \vsingle^t$), we can write that
\begin{align*}
\ell\left(\vsingle^t,\caught^t; \lambda\right) &= \log\left(1-F\left(\vsingle^t-1;\lambda\right)\right) 
= \log\left(\sum_{k=\vsingle^t}^{\infty} \frac{\lambda^{k} e^{-\lambda}}{k!}\right) = 
-\lambda + \log\left(\sum_{k=\vsingle^t}^{\infty} \frac{\lambda^k}{k!}\right).
\end{align*}
Again taking the derivative, we get
\begin{align*}
\frac{\partial  \ell}{\partial \lambda}=
-1 + \frac{\frac{\partial}{\partial \lambda} \sum_{k=\vsingle^t}^{\infty} \frac{ \lambda^k}{k!}}{\sum_{k=\vsingle^t} \frac{ \lambda^{k}}{k!}} 
&= -1 + \frac{\frac{\partial}{\partial \lambda}\left[ e^{\lambda} - \sum_{k=0}^{\vsingle^t-1} \frac{\lambda^k}{k!}\right]}{\sum_{k=\vsingle^t}^{\infty} \frac{ \lambda^{k}}{k!}}
= -1 + \frac{\lambda - \sum_{k=1}^{\vsingle^t-1} \frac{\lambda^{k-1}}{(k-1)!}}{\frac{ \lambda^{k}}{k!}}.
\end{align*}
	
The fraction is the quotient of two continuous functions and the denominator is 
nonzero for any $\lambda \in [\lambda_{\min},\lambda_{\max}]$. 
So $\partial \ell /\partial \lambda$ is continuous in $\lambda$. Since the image of a 
continuous function on a compact set remains compact, $\partial \ell /\partial \lambda$ is 
bounded for all $t$. Thus $\ell(\vsingle^t,\caught^t;\lambda)$ is Lipschitz-continuous in this case as well. 
\end{proof}
\fi

\ifproofs
\else
\begin{lemma}
\label{lem:lipschitz}
For any $\lambda$, $\lambda' \in [\lambda_{\min},\lambda_{\max}]$ such 
that $|\lambda-\lambda'|<\epsilon$, we have that 
$|\ell(\vsingle^t, \caught^t; \lambda)-\ell(\vsingle^t, \caught^t; \lambda')| \le b \epsilon$ 
for some constant $b$.
\end{lemma}

\fi

\ifproofs
\else
\begin{proof}[Proof of Lemma~\ref{thm:uniform}]
	Define the $\epsilon$-net as 
$N_\epsilon=\{\lambda_{\min}, \lambda_{\min}+\epsilon, \lambda_{\min}+2\epsilon, \dots, \lambda_{\max}\}.$ 
We use $k= |N_\epsilon|$ to denote the cardinality of the set; so 
$k=\lceil \lambda_{\max} - \lambda_{\min}\rceil/\epsilon$.  Note that 
for any $\lambda \in [\lambda_{\min}, \lambda_{\max}]$, there exists $\lambda' \in N_\epsilon$ 
such that $|\lambda- \lambda'| \le \epsilon$.

By Corollary \ref{cor:unionbound}, for $\Lambda=\{\lambda_1, \lambda_2, \dots, \lambda_k \}$, 
with probability $1-\delta$,
\begin{align}\label{ineq:supbound} \max_{\lambda \in \Lambda} 
\left|\hat{\like}\left(\hvec^t, \lambda\right) - \like^*\left(\hvec^t, \lambda\right)\right| \le  
\sqrt{\frac{2C^2}{t} \ln\left(\frac{2k\R}{\delta}\right)}.
\end{align}

Now, for any $\lambda \in [\lambda_{\min},\lambda_{\max}]$ by triangle inequality we have that
\begin{align*}
\left|\hat{\like}\left(\hvec^t, \lambda\right) - \like^*\left(\hvec^t, \lambda\right)\right|
\le &\left|\hat{\like}\left(\hvec^t, \lambda\right) - \hat{\like}\left(\hvec^t, \lambda'\right)\right|  
+\left|\hat{\like}\left(\hvec^t, \lambda'\right) - \like^*\left(\hvec^t, \lambda'\right)\right| 
+\left|\like^*\left(\hvec^t, \lambda'\right) - \like^*\left(\hvec^t, \lambda\right)\right|.
\end{align*}
By Lemma \ref{lem:lipschitz}, the first and third term are at most $\epsilon b$ where $b$ 
is again the Lipschitz constant in Lemma \ref{lem:lipschitz}. 
Applying this to the closest $\lambda_k \in N_{\epsilon}$ and noting that the inequality in 
Equation~\ref{ineq:supbound} binds on the middle term with $C$ as in 
Lemma \ref{lem:probabilisticsinglelambda}, we have
 \begin{align*}
 \left|\hat{\like}\left(\histt, \lambda\right) - \like^*\left(\histt, \lambda\right)\right| & 
 \le \epsilon b + \sqrt{\frac{2C^2\ln\left(\frac{2k\R}{\delta}\right)}{t}} + \epsilon b 
\le 2 \epsilon b + \sqrt{\frac{2C^2\ln\left(\frac{2k\R}{\delta}\right)}{t}}
\le 2 \epsilon b + \sqrt{\frac{2C^2\ln\left(\frac{2 \R\lceil \lambda_{\max} - 
\lambda_{\min}\rceil }{\epsilon \delta}\right)}{t}}.
\end{align*}

Setting $\epsilon=1/t$ yields the claim. Note that as $t \to \infty$, the difference 
approaches not only constant but diminishes to $0$.
\end{proof}
\fi

\begin{lemma}
\label{lem:closeness} Suppose that a continuous function $g(x): [a,b] \mapsto \mathbb{R}$ has a unique 
maximizer $x^*$. Then for every $\epsilon >0, \ \exists \eta>0$ such that $g(x^*)-g(x) < \eta$ implies 
$|x-x^*| < \epsilon$.  In particular, this $\eta$ can be written as $$g(x^*)-\underset{x \in [a, x^*-\epsilon] \cup [x^*+\epsilon, b]}{\max}g(x).$$ 
When $g$ is concave and differentiable, $\eta$ can be evaluated by evaluating $g$ at a constant number of points.
\end{lemma}
\begin{proof}
Let $X_\epsilon$ be the $\epsilon$-radius open ball centered at $x^*$, 
and let $\Theta$ be $[a,b]\setminus X_\epsilon$ i.e. the domain of $g$ 
excluding the $\epsilon$-radius ball centered at the maximizer.  
Since $X_\epsilon$ is open, $\Theta$ is closed and bounded, and therefore compact.  
Since $g$ is continuous, the restriction of $g$ to $\Theta$ has some maximum 
$g(\hat{x})$ for some (not necessarily unique) $\hat{x}\in \Theta$.

Observe that, if for any $x\in [a,b]$, we have that $g(x)>g(\hat{x})$, then $x$ 
must be in $X_\epsilon$. Otherwise $\hat{x}$ would not be a maximizer of the 
restriction of $f$ to $\Theta$.  Choose $\eta = g(x^*)-g(\hat{x})$.  Then, 
because $g(x)>g(\hat{x}) $, we have that $g(x^*)-g(x) < g(x^*)-g(\hat{x})=\eta$.
Therefore, $|g(x^*)-g(x)|<\eta$ implies $|x^*-x|<\epsilon$ , completing the proof of existence.  

The dependence of $\eta$ on $\epsilon$ is function-dependent,  but by construction, $\eta(\epsilon)$ can 
be computed 
by taking the maximum of $g$ on $\Theta \setminus X_{\epsilon}$ and subtracting it from $g(x^*)$. Notice that 
in the case of concavity and differentiability, this maximization problem is easy to calculate. If $g$ is concave and 
differentiable over $[a,b]$, its restrictions to $[a,x^*-\epsilon]$ and $[x^*+\epsilon,b]$ are as well. A differentiable, 
concave function on an interval can only be maximized at an interior critical point or at one of the two end points. 
Hence if $x^*$ is interior, it must be a critical point, and by concavity it is the unique critical point on $[a,b]$, so 
$g$ can have no critical points on $\Theta \setminus X_{\epsilon}$. Thus $g$ restricted to $[a,x^*-\epsilon]$ is 
maximized at either $a$ or $x^*-\epsilon$; similarly for $g$ restricted to $[x^*+\epsilon,b]$. On the other hand, 
if $x^*$ is either $a$ or $b$, there is just one interval in $\Theta \setminus X_{\epsilon}$ to check, and checking 
the endpoints of that interval plus exhaust the possible maximizers. In either case, no more than 4 points need 
be checked; in contrast, without concavity or differentiability, finding the maximum on $g$ on 
$\Theta \setminus X_{\epsilon}$ could require more involved optimization techniques. 

\end{proof}

\begin{cor}
\label{cor:closeness}
For any fixed $\hvec^t$ and $\lambda \in [\lambda_{min}, \lambda_{max}]$, the following 
must hold true for $\like^*(\hvec^t, \lambda)$ whose unique maximizer is $\lambda^*$. 
For every $\epsilon,\ \exists \eta > 0$ such that $\like^*(\hvec^t, \lambda^*) - \like^*(\hvec^t, \lambda) < \eta$ 
implies $|\lambda - \lambda^*| < \epsilon$, where 
$\eta$ is $$\like^*(\hvec^t, \lambda^*)-\underset{\lambda \in [\lambda_{min}, \lambda^*-\epsilon] 
\cup [\lambda^*+\epsilon, \lambda_{max}]}{\max}\like^*(\hvec^t, \lambda).$$
\end{cor}

Fixing any group, we show that  for any fixed allocation $\vsingle$, the difference between 
the discovery probability with respect to the true and estimated candidate distributions
 is proportional to the total variation distance between the true and estimated
distributions.

\begin{lemma}
\label{lem:disc-prob}
Let $\vsingle$ be any fixed allocation to the group. Then
$\left|\f(\vsingle, \C(\lambda^*)) - \f(\vsingle, \C(\hat{\lambda}))\right| \le 2D_{TV}(\C(\lambda^*), \C(\hat{\lambda})).$
\end{lemma}
\begin{proof}
\begin{align*}
\left|\f(\vsingle, \C(\lambda^*)) - \f(\vsingle, \C(\hat{\lambda}))\right| 
&= \left|\E_{c\sim\C(\lambda^*)}\left[\frac{\min(\vsingle,c)}{c}\right] - \E_{\c\sim\C(\hat{\lambda})}\left[\frac{\min(\vsingle,c)}{c}\right]\right|
= \left|  \sum_{c=0}^\infty \frac{\min(\vsingle,c)}{c} \left(\Pr[c;\lambda^*] - \Pr[c;\hat{\lambda}]\right) \right|\\
&\leq  \sum_{c=0}^\infty \frac{\min(\vsingle,c)}{c} \left| \Pr[c;\lambda^*] - \Pr[c;\hat{\lambda}] \right|
\leq \sum_{c=0}^\infty \left|\Pr[c;\lambda^*] - \Pr[c;\hat{\lambda}] \right|\\
&\leq 2D_{TV}(\C(\lambda^*), \C(\hat{\lambda})).
\end{align*}
\end{proof}
\section{Omitted Details from Section~\ref{sec:exp}}
\label{sec:omitted-details-exp}
Table~\ref{table:avg} represents the average and standard
 deviation of the number of daily reported incidents in all of the districts
 in the Philadelphia Crime Incidents dataset.
\begin{table}[ht!]
\centering
\begin{tabular}{|c|c|c|} \hline
id & average  & standard deviation \\ \hline
1   &     11.35    &    5.1   \\ \hline
 2  &     27.44    &    9.24   \\ \hline
3   &     20.37   &     9.27   \\ \hline
5   &     7.36    &   3.65   \\ \hline
6    &    22.67    &    7.54   \\ \hline
7   &    10.47     &   4.56   \\ \hline
8   &     17.26    &   6.77   \\ \hline
 9  &     19.83     &   7.15   \\ \hline
 12 &      30.97    &    10.86  \\ \hline
14    &    28.69    &   9   \\ \hline
15   &     43.5     &   12.69   \\ \hline
16  &      17.36     &   6.99   \\ \hline
17   &     17.41     &   7.45   \\ \hline
 18  &     25.88   &    8.35   \\ \hline
  19  &    33.43    &    10.71   \\ \hline
 22  &     30.45    &    9.89	   \\ \hline
 24 &      38.47    &    11.82   \\ \hline
 25  &     35.54   &     12.41   \\ \hline
 26  &     20.55   &     7.16   \\ \hline
  35  &    30.92  &     9.79   \\ \hline
  39  &    23.24   &    7.16   \\ \hline
\end{tabular}
\caption{The average and standard deviation in the ground truth distributions
in each of the districts of the Philadelphia Crime Incidents dataset.\label{table:avg}}
\end{table}

Table~\ref{table:fit} displays the $\ell_1$ and $\ell_{\infty}$ distances 
of the ground truth and best Poisson fit distribution for all of the districts in 
the Philadelphia Crime Incidents dataset. Observe that the $\ell_{\infty}$ metric
shows that the Poisson fit provides a close approximation for the ground
truth distribution. Also as Figure~\ref{fig:dists-philly-fitting} displays, the Poisson 
fit is a better approximation to the ground truth had we ignored the 0 counts (the frequency
of the days in which no crime has been reported) in the dataset. This metric has also been measured
in Table~\ref{table:fit} in the ``no zero" columns. Note that the goodness of fit would
improve significantly when removing the 0 counts according to the $\ell_1$ measure but would
not change at all according to the $\ell_{\infty}$ measure.
\begin{table}[ht!]
\centering
\begin{tabular}{|c|c|c|c|c|} \hline
id & $\ell_1$ & $\ell_1$ (no zero) & $\ell_{\infty}$ & ${\ell_\infty}$ (no zero)\\ \hline
1   &     0.1656    &    0.1562  &      0.0315   &     0.0315\\ \hline
 2  &     0.2021    &    0.1928   &     0.0293   &     0.0293\\ \hline
3   &     0.3420   &     0.3327  &      0.0493   &     0.0493\\ \hline
5   &     0.1203    &    0.1093  &      0.0328  &      0.0328\\ \hline
6    &    0.1853    &    0.1760  &      0.0309  &      0.0309\\ \hline
7   &     0.1269     &   0.1175  &      0.0279  &      0.0279\\ \hline
8   &     0.1835     &   0.1742  &      0.0365   &     0.0365\\ \hline
 9  &     0.2025     &   0.1931   &     0.0315    &    0.0315\\ \hline
 12 &      0.2600    &    0.2507  &      0.0304   &     0.0304\\ \hline
14    &    0.2024     &   0.1931  &      0.0239   &     0.0239\\ \hline
15   &     0.2305     &   0.2212  &      0.0276   &     0.0276\\ \hline
16  &      0.2024     &   0.1931  &      0.0354   &     0.0354\\ \hline
17   &     0.2495     &   0.2402  &      0.0436   &     0.0436\\ \hline
 18  &     0.1953    &    0.1860  &      0.0278   &     0.0278\\ \hline
  19  &    0.2494    &    0.2401  &      0.0337   &     0.0337\\ \hline
 22  &     0.2469    &    0.2375  &      0.0326   &     0.0326\\ \hline
 24 &      0.2529    &    0.2436   &     0.0284   &     0.0284\\ \hline
 25  &     0.2844   &     0.2751   &     0.0312   &     0.0312\\ \hline
 26  &     0.1896   &     0.1803  &      0.0328   &     0.0328\\ \hline
  35  &    0.2187   &     0.2095   &     0.0291   &     0.0291\\ \hline
  39  &    0.1478   &     0.1385   &     0.0243   &     0.0243\\ \hline
average  &  0.2123   &     0.2029   &   0.0319   &      0.0319\\ \hline
\end{tabular}
\caption{\label{table:fit}Various statistical distances between the ground truth distribution 
and best Poisson fit for each of the districts.}
\end{table}

In Figure~\ref{fig:Paretotrue} we compare the Pareto frontiers for the optimal and fitted curves.
Figure~\ref{fig:Paretotrue} shows that the performance (in terms of the utility/fairness trade-off) 
does not degrade significantly when we assume the Philadelphia Crime Incidents dataset 
is generated according to Poisson distributions.

\begin{figure}[ht!]
\centering
\includegraphics[width=9.5cm]{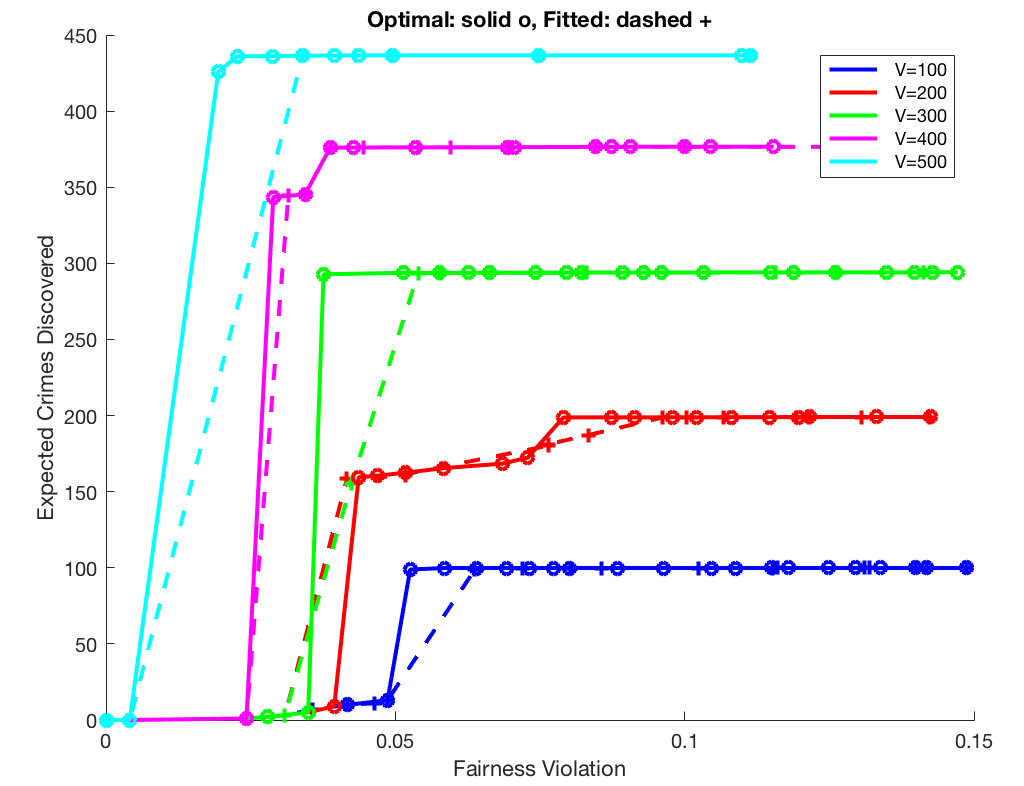}
\caption{\label{fig:Paretotrue}
Pareto frontier of the expected crimes discovered versus fairness violation.}
\end{figure}
\section{Omitted Details From Section \ref{sec:urn}}
\label{sec:omitted-details-urn}
First, we write down the integer programming to compute an optimal $\alpha$-fair 
allocation in the random model described in Section~\ref{sec:urn-opt-fair}. 
\begin{equation*}
\begin{aligned}
& \underset{\v=\{v_1,\ldots, v_\R\} }{\max} && \sum_{i=1}^\R \frac{\vi \E[\ci]}{\m_i},  \\
& \text{subject to} && \left| \frac{\vi}{\m_i} -\frac{\vj}{\m_j} \right| \le \alpha, \forall i\text{ and } j,\\
&  &&\sum_{i=1}^\R \vi \le \V,\\
&  &&\vi \in \mathbb{N}, \forall i.
\end{aligned}
\end{equation*}
Next we present the proof of Theorem~\ref{thm:urn-pof}.
\begin{proof}[Proof of Theorem~\ref{thm:urn-pof}]
Observe that when $\V/\m_1\leq \alpha$ the allocation that sends all of the units of the resource to group 1
is both optimal and $\alpha$-fair.
For the case that $\V/\m_1 >\alpha$, 
we will first provide an upper bound on the PoF by providing some allocation $\v$, 
which is is $\alpha$-fair, and use that to show fairness does not deteriorate the 
total number of candidates discovered by $\v$ compared to an optimal $\alpha$-fair allocation. 
And then, we will construct a specific candidate distributions $\C'$ and compute the PoF 
to show that the upper bound is tight.

Consider the following allocation $\v$.
\begin{equation}
\label{eq:v}
\vi = 
	\begin{cases} 
	\left(\frac{\V+\alpha (M-m_1)}{M}\right)\m_1, & i = 1, \\
	\left(\frac{\V-\alpha \m_1}{M}\right) \m_i, & \text{otherwise}.
	\end{cases}
\end{equation}
We show that $\v$ is a feasible $\alpha$-fair allocation.
To show feasibility, observe that
\begin{align*}
\sum_{i \in [\R]} \vi &=\frac{\V-\alpha m_1}{M}\sum_{i \in [\R]} \m_i +\alpha m_1 = \V.
\end{align*}
To show that $\v$ is $\alpha$-fair observe that
\begin{align*}
\left|\f_1(v_1) - \f_i(\vi)\right| &= \left|\frac{v_1}{\m_1}-\frac{\vi}{\m_i}\right| = 
\left|\frac{\V+\alpha (M-m_1)}{M}-\frac{\V-\alpha \m_1}{M}\right| = \alpha \text{ and } \left|\f_i(\vi) - \f_j(\vj)\right| = 0, \forall i,j \neq 1.
\end{align*}

Since $\v$ is a feasible $\alpha$-fair allocation, for any $\C$, we have that 
$$\util(\vopt^\alpha, \C) \ge \util(\v, \C) \ge \mu_1 v_1,$$ 
where the last inequality is derived by counting only the candidates
that allocation $\v$ discovers in group $1$ according to the random model and ignoring all the 
discoveries in other groups.
Moreover, for any $\C$, $\util(\vopt^*, \C) = \mu_1 \V$ by the argument in Section~\ref{sec:urn-opt}. 
Therefore, 
\begin{align*}
\text{PoF} &= \max_{\C} \frac{\util(\vopt^*, \C)}{\util(\vopt^{\alpha}, \C)}
\le \frac{\mu_1 \V}{\mu_1 v_1} 
= \frac{\V M}{\m_1\left(\V+\alpha (M-m_1)\right)}\leq \frac{M}{m_1+\alpha(M-m_1)},
\end{align*}
where the last inequality uses the assumption that $\V\leq m_1$.

To derive the lower bound on the PoF, we construct a candidate distribution $\C'$ 
and compute the PoF.
We show that our lower on PoF matches 
the upper bound, so our analysis on PoF is tight.

To construct $\C'$ assume all groups have size $\V$ i.e. $\m_i = \V$ for all $i$. Furthermore, assume 
group 1 has $\V$ candidates and the rest of the groups have 0 candidates deterministically. 
Then the optimal allocation $\vopt^*$ is to send all the $\V$ units of resource to group 1, and doing so
will discover $\V$ candidates (since $\mu_1 = 1$). As for the optimal $\alpha$-fair allocation, we show 
that $\vopt^\alpha = \v$ where $\v$ 
is the same allocation as the allocation used in our upper bound (see Equation~\ref{eq:v}). 

\begin{lemma}
For $\C'$, $\vopt^\alpha = \v$ where $\v$ is
defined in Equation~\ref{eq:v}.
\end{lemma}
\begin{proof}
For $\C'$, because all groups have the same size, $v_i = v_j$ for all $i,j \neq 1$.
Since $\Sigma_{i\in[\R]}\vi = \V$, $v_i = (\V - v_1)/(\R-1)$ for all $i \neq 1$. 
Now,  any feasible $\alpha$-fair allocation $\v'$  must have $v'_1 \le v_1$. 
Assume by the way of contradiction that $v'_1 > v_1$.
 Then, after assigning $v'_1$ units of resource to group 1, the remaining 
units that will be strictly less than $\V - v_1$ will be distributed among the remaining $\R-1$ 
groups. By the pigeonhole principle, there must exist at least one group
$j$ such that 
$$v'_{j} < \frac{\V-v_1}{\R-1} = v_j = \left(\frac{\V-\alpha \m_1}{M}\right) \m_j.$$ 
Now observe that 
\begin{align*}
\left|\f_1(v'_1) - \f_{j}(v'_{j})\right| &= \left|\frac{v'_1}{m_1} - \frac{v'_{j}}{\m_j}\right|>  
\left|\frac{\V+\alpha (M-m_1)}{M}-\frac{\V-\alpha \m_1}{M}\right| = \alpha.
\end{align*}
So $\v'$ cannot be $\alpha$-fair.
Therefore, $\v$ must be an optimal $\alpha$-fair  allocation since 
in $\v$ the maximum number of units of resources are allocated
to group 1 which is the only group that contains candidates.
\end{proof}

Note that the number of candidates discovered by $\vopt^\alpha$ is 
exactly $v_1$ since $\mu_1=1$.
So given the $\vopt^\alpha$ for $\C'$ we can compute the PoF as follows.
\begin{align*}
\text{PoF} &= \frac{\util(\vopt^*, \C')}{\util(\vopt^{\alpha}, \C')} = \frac{\V}{v_1}  
= \frac{\V M}{\left(\V+\alpha (M-\m_1)\right)\m_1}=
\frac{\m_1 M}{\left(\m_1+\alpha (M-\m_1)\right)\m_1}=\frac{M}{m_1+\alpha(M-m_1)}.
\end{align*}
This lower bounds matches our upper bound so our analysis on PoF is tight.
\end{proof}
\subsection{Relaxing the Assumption of $\V\leq \m_i$.}
\label{sec:omitted-details-urn-assumpt}
In this section we relax the assumption that $\V\leq \m_i$ for all groups $i$.
We first show how an optimal allocation can be computed using a greedy algorithm.
Recall that we have assumed $\mu_1\geq \mu_2 \geq \ldots \geq \mu_\R$.
Optimal algorithm allocates
$\vsingle_1=\min(\V, \m_1)$ units of resource to group 1. And recurse with the 
remaining $\V-\vsingle_1$
resources on the rest of the groups. If this algorithm allocates resources to groups
1 through $k$ then the expected utility of the algorithm can be written as 
$$\sum_{i=1}^{k-1} \mu_i \m_i + \mu_k\vsingle_k=
\sum_{i=1}^{k-1} \mu_i \m_i + \mu_k(\V-\sum_{i=1}^{k-1}m_i).$$
In the case that $\V\leq \m_i$, the algorithm allocates all the resources to
group 1 without any leftover resources for other groups.

We note that an optimal $\alpha$-fair allocation can still be computed with the 
same integer program. Furthermore, the lower bound on the PoF as computed 
in Theorem~\ref{thm:urn-pof} continues to hold even when we relax the assumption. 

However, our upper bound analysis on the PoF breaks when we relax the assumption. 
While it is possible to derive a similar upper bound by careful analysis, we do not 
investigate this direction as the lower bound on PoF shows that the PoF in the 
random model can still be quite high even without the assumption made in Section~\ref{sec:min-model}.
\end{document}